\documentclass[]{bytedance_seed}

\usepackage[toc,page,header]{appendix}

\usepackage{minitoc}
\usepackage{url}            %
\usepackage{booktabs}       %
\usepackage{multirow}    
\usepackage{amsfonts}       %
\usepackage{nicefrac}       %
\usepackage{microtype}      %
\usepackage{natbib}
\usepackage{enumerate}
\usepackage{hhline}
\usepackage{makecell}
\usepackage{pifont}
\usepackage{wrapfig}

\usepackage{graphicx} %
\usepackage{amsmath}
\usepackage{amsthm}
\usepackage{amssymb}
\usepackage{tikz}
\usepackage{xcolor}
\usetikzlibrary{arrows}

\allowdisplaybreaks

\usepackage{mathrsfs}

\usepackage{hyperref}
\usepackage{bm}

\usepackage{tcolorbox}

\newtheorem{thm}{Theorem}
\crefname{thm}{Theorem}{Theorems}

\newtheorem{assumption}{Assumption}
\crefname{asm}{Assumption}{Assumptions}

\newtheorem{lemma}{Lemma}

\newtheorem{corollary}{Corollary}
\crefname{corollary}{Corollary}{Corollaries}

\title{Visual Generation Unlocks Human-Like Reasoning through Multimodal World Models}

\author[1,2,*]{Jialong Wu}
\author[2, \dagger]{Xiaoying Zhang}
\author[2]{Hongyi Yuan}
\author[1,2,*]{Xiangcheng Zhang}
\author[1]{Tianhao Huang}
\author[1]{Changjing He}
\author[1,2,*]{Chaoyi Deng}
\author[2]{Renrui Zhang}
\author[2]{Youbin Wu}
\author[1, \dagger]{Mingsheng Long}

\affiliation[1]{Tsinghua University}
\affiliation[2]{ByteDance Seed}

\contribution[*]{Work done at ByteDance Seed}
\contribution[\dagger]{Corresponding authors}

\abstract{
Humans construct internal models of the world and reason by manipulating the concepts within these models. Recent advances in artificial intelligence (AI), particularly chain-of-thought (CoT) reasoning, approximate such human cognitive abilities, where world models are believed to be embedded within large language models. Expert-level performance in formal and abstract domains such as mathematics and programming has been achieved in current systems, which rely predominantly on verbal reasoning as their primary information-processing pathway. However, they still lag far behind humans in domains like physical and spatial intelligence, which require richer representations and prior knowledge. The emergence of unified multimodal models (UMMs) capable of both verbal and visual generation has therefore sparked interest in more human-like reasoning grounded in complementary multimodal pathways, though a clear consensus on their benefits has not yet been reached. From a world-model perspective, this paper presents the first principled study of when and how visual generation benefits reasoning. Our key position is the \textit{visual superiority hypothesis}: for certain tasks--particularly those grounded in the physical world--visual generation more naturally serves as world models, whereas purely verbal world models encounter bottlenecks arising from representational limitations or insufficient prior knowledge. Theoretically, we formalize internal world modeling as a core component of deliberate CoT reasoning and analyze distinctions among different forms of world models from both informativeness and knowledge aspects. Empirically, we identify and design tasks that necessitate interleaved visual-verbal CoT reasoning, constructing a new evaluation suite, \textit{VisWorld-Eval}. Through controlled experiments on a state-of-the-art UMM, we show that interleaved CoT significantly outperforms purely verbal CoT on tasks that favor visual world modeling. Conversely, it offers no clear advantage for tasks that do not require explicit visual modeling. Together, these insights and findings clarify the applicability and potential of multimodal world modeling and reasoning for more powerful, human-like multimodal AI. We publicly release our evaluation suite to facilitate further research.
}

\date{\today}
\checkdata[Project Lead]{Jialong Wu at \email{wujialong0229@gmail.com}}
\correspondence{Mingsheng Long at \email{mingsheng@tsinghua.edu.cn}, \\ {\hspace*{2.4cm} Xiaoying Zhang at \email{zhangxiaoying.xy@bytedance.com}}}

\checkdata[Project Page]{\url{https://thuml.github.io/Reasoning-Visual-World}}

\begin{document}
\maketitle

\newpage

\section{Introduction}

\begin{figure}[tbp]
  \centering
  \includegraphics[width=0.9\linewidth]{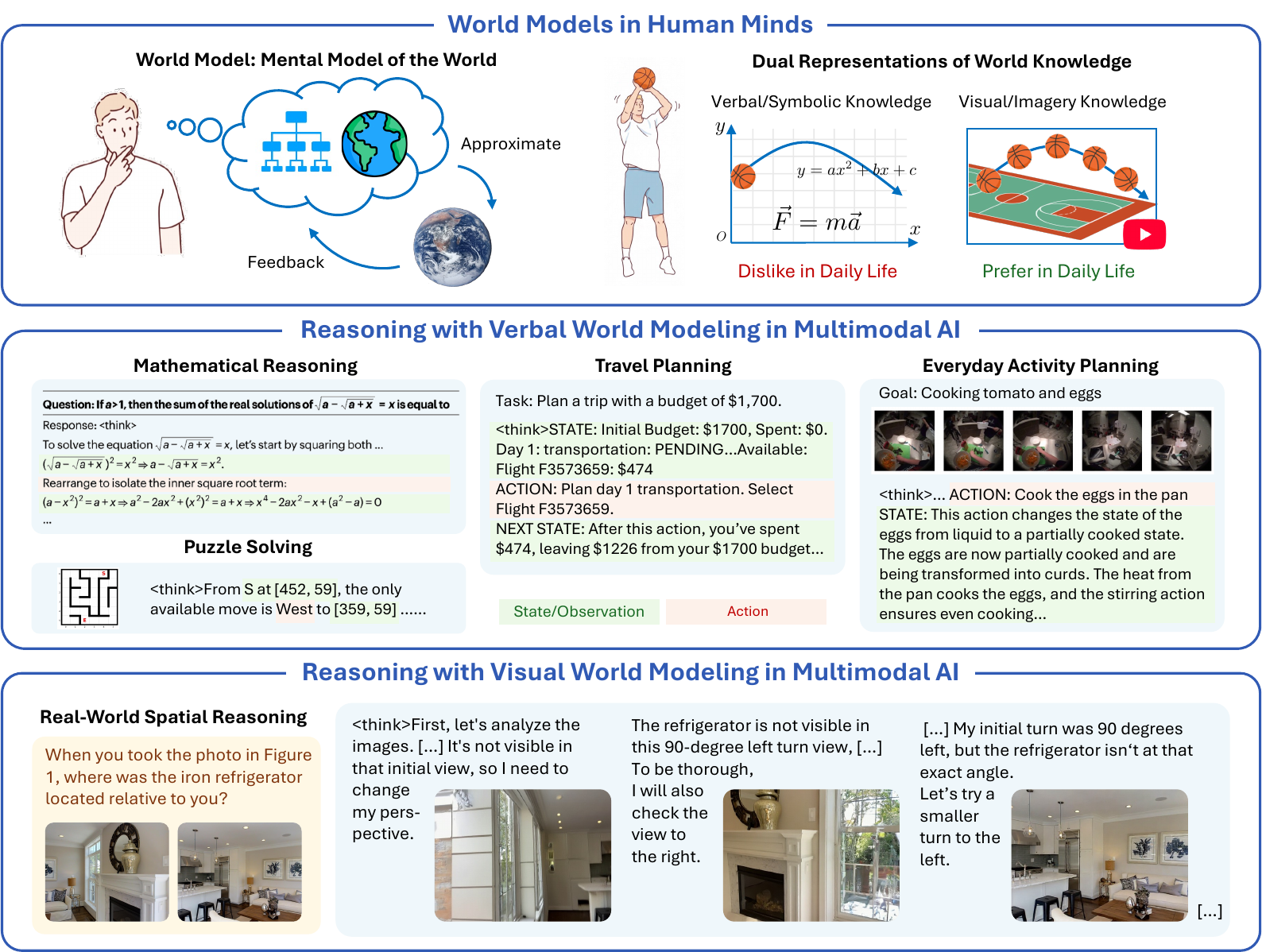}
  \caption{Overview of a world-model perspective on multimodal reasoning. (a) Humans construct mental models of the world, representing information and knowledge through two complementary channels--verbal and visual--to support reasoning, planning, and decision-making. (b) Recent advances in large language models (LLMs) and vision language models (VLMs) largely rely on verbal chain-of-thought reasoning, leveraging primarily verbal and symbolic world knowledge. (c) Unified multimodal models (UMMs) open a new paradigm by using visual generation for visual world modeling, advancing more human-like reasoning on tasks grounded in the physical world. Examples of reasoning with verbal world modeling are adapted from \citet{guo2025deepseek,du2025revisiting,chen2025planning,zhang2025agent}.}
  \label{fig:concept}
\end{figure}

Humans construct internal mental models of the external world that represent objects and concepts, along with their relationships, structures, and operational mechanisms \cite{craik1967nature, forrester1971counterintuitive}. These models support reasoning and decision-making by enabling mental simulation, allowing individuals to anticipate the outcome of actions without actually taking them  \cite{loomis1991mental}. For example, if a glass of water is spilled on the table, people can rapidly mentally simulate the ensuing events: the water falling downward, spreading across the surface, and potentially dripping onto the floor. Such predictions lead them to quickly move valuable items away or reach for a towel. Beyond physical systems, mental models also extend to any domain where relational structures can be simulated, such as mathematics and logic \cite{johnson1983mental, lakoff2000mathematics}, making them fundamental to how humans understand and interact with all aspects of the world. 

Cross-disciplinary researchers in philosophy, psychology, cognitive science, and related fields have a long history of developing computational models of human mental models \cite{norman2014some}. Among them, artificial intelligence (AI) shares a core ambition of building machines that reason like people. Although debates remain, recent breakthroughs, especially in large language models (LLMs) and chain-of-thought (CoT) reasoning, have made a substantial step towards approximating human reasoning grounded in \textit{mental models of the world}, often referred to as \textbf{world models} \cite{ha2018world, lecun2022path} in the AI literature. During chain-of-thought reasoning, LLMs explore, reflect, and backtrack within the structured solution space, guided by world knowledge acquired through large-scale pre-training. These capabilities have already driven progress in diverse domains, including programming \cite{guo2025deepseek}, mathematics \cite{trinh2024solving, guo2025deepseek}, scientific discovery \cite{swanson2025virtual}, clinical medicine \cite{tu2025towards}, and robotics \cite{mon2025embodied}.

Such reasoning capabilities have also been extended to multimodal AI systems, particularly vision language models (VLMs) \cite{hurst2024gpt, bai2025qwen2, guo2025seed1, yao2025efficient}. These systems typically incorporate visual inputs by aligning visual representations with the embedding space of LLMs, resulting in reasoning that remains primarily constrained to a linguistic space. In contrast, human mental models operate over multiple forms of mental representations. Dual-coding theory \cite{paivio1990mental} suggests that the mind processes information through two complementary codes: verbal and imagery (particularly visual) representations. These pathways can function independently but often collaborate to support reasoning. Indeed, visual imagery has been shown to have advantages over words in memory encoding and retrieval \cite{landy2007abstract}; and individuals with aphantasia, who lack the ability to visualize mental imagery, exhibit worse performance on tasks such as visual search \cite{monzel2024s}. These evidence from psychology and cognitive science therefore suggest that the absence of a dedicated visual information pathway may explain why current multimodal AI systems excel in formal and abstract domains dominated by verbal world knowledge, yet continue to fall far short of human performance on tasks involving physical and spatial reasoning \cite{schulze2025visual, cai2025has}, which fundamentally depend on \textbf{visual world modeling}.

Next-generation multimodal AI systems are evolving to be built upon unified multimodal models (UMMs) \cite{team2024chameleon, wu2025janus, wang2024emu3, deng2025emerging}, which seamlessly integrate both verbal and visual generation capabilities. The newly introduced visual generation component offers the potential to explicitly realize visual world modeling, a critical element of \textbf{multimodal world models} in human-like reasoning that current systems largely lack. This naturally makes us ponder: \textit{Can current UMMs truly leverage their visual generation capability to enhance reasoning and thereby narrow the performance gap between multimodal AI and humans?} A growing body of preliminary research \cite{li2025imagine, zou2025uni, liang2025rover, zhou2025visualizing, gu2025thinkmorph} has begun exploring this question from different perspectives. However, the findings so far remain inconclusive. Reported empirical results are mixed, showing no consistent trends that visual generation reliably improves reasoning performance. Moreover, the evaluation tasks used in current studies are designed heuristically, lacking a principled basis for understanding when and how visual generation can meaningfully contribute to multimodal reasoning.

In this paper, we present the first principled study of when and how visual generation benefits reasoning from a \textbf{world-model perspective} (see Figure~\ref{fig:concept}), making both theoretical and empirical contributions.

Theoretically, we rigorously bridge the concepts of world models and reasoning. (1) \textbf{World model formulations}: We formulate multimodal world models to approximate the underlying \textit{multi-observable Markov decision processes} (MOMDP) of tasks, and define two fundamental capabilities of world models, namely \textit{world reconstruction} and \textit{world simulation}. (2) \textbf{World model-based reasoning}: To realize world models for reasoning, we introduce three reasoning formulations. Two rely solely on verbal CoTs through \textit{implicit or verbal world modeling}, while the third interleaves verbal and visual CoTs that explicitly incorporate visual generation as a form of \textit{visual world modeling}. (3) \textbf{The visual superiority hypothesis}: Under this framework, we analyze the distinctions among different world models, highlighting the richer informativeness and complementary prior knowledge afforded by visual world modeling. These insights motivate our central hypothesis that visual world modeling is superior for certain tasks, particularly those grounded in the physical world.

Empirically, we validate these insights through a series of controlled experiments. (4) \textbf{The VisWorld-Eval suite}: We identify and design tasks that specifically isolate and demand each atomic world model capability, forming a new evaluation suite to facilitate future research. This suite, \textit{VisWorld-Eval}, collects seven tasks spanning both synthetic and real-world domains. (5) \textbf{Empirical evaluation}: Experiments with a state-of-the-art UMM \cite{deng2025emerging} on VisWorld-Eval reveal findings consistent with our insights and theoretical analysis. In tasks where verbal world modeling suffers from representational bottlenecks or insufficient prior knowledge, interleaved CoT delivers substantial performance improvements. By contrast, it offers no clear advantages in tasks such as mazes and Sokoban, whose simple states do not require explicit visual world modeling. We further conduct dedicated analyses, including evidence revealing emergent implicit world modeling in the maze task.

We hope this work provides early evidence for the central role of multimodal world models in general-purpose AI, in which complementary verbal and visual knowledge emerge from generative modeling across modalities, with the latter being especially valuable for bringing human-like intelligence into the physical world.

\section{Related Work}

\textbf{World models.} The field of world models is rapidly evolving, yet remains far from reaching consensus on definitions or methodologies. Although psychology and cognitive science suggest that human mental models rely on compact representations that discard irrelevant details, how to scale approaches capable of learning such abstract representations \cite{schrittwieser2020mastering, hansen2022temporal, lecun2022path} to arbitrary domains and modalities is still unclear. Consequently, most current techniques preserve complete information of observations, either through reconstructable latent representations  \cite{ha2018world, hafner2025mastering} or directly at the level of raw data. Prominent examples include modern video generation world models \cite{genie3, agarwal2025cosmos, alonso2024diffusion, wu2024ivideogpt} which capture concrete pixel-level dynamics. In contrast, language inherently provides a higher level of abstraction, making it more similar to human mental representations \cite{wang2024can, wu2025rlvr, wang2025vagen, yin2025spatial, chen2025planning}. This motivates the promise of unified multimodal models that generate both languages and visuals as a new direction for building more human-like world models.

\textbf{Unified multimodal models.} Multmodal understanding \cite{hurst2024gpt, bai2025qwen2, guo2025seed1} and visual generation \cite{rombach2022high, seedream2025seedream} have long developed in isolation. Recently, there has been growing interest in integrating these two capabilities into a single unified model. This can be straightforwardly achieved by forwarding the representations of vision language models to an external visual generation module \cite{tong2025metamorph, pan2025transfer}. A more unified approach is to model both language and visual modalities within a single backbone. While language is predominantly modeled through autoregressive next-token prediction, the design space of visual modalities spans a wide spectrum, from discrete tokenization with autoregressive \cite{wang2024emu3, team2024chameleon, wu2025janus} or masked modeling \cite{xie2024show, guo2025can}, to continuous tokenization with diffusion or flow-based modeling \cite{zhou2024transfusion, ma2025janusflow, deng2025emerging}. Among these efforts, BAGEL \cite{deng2025emerging} is one of the most widely adopted open-source models achieving state-of-the-art performance. Despite substantial progress in building unified multimodal models (UMMs), existing evaluations still primarily assess their understanding and generation capabilities separately. One widely recognized advantage of UMMs lies in leveraging reasoning abilities of handling complex instructions to enhance visual generation or editing \cite{zhao2025envisioning, guo2025thinking}. Yet when and how visual generation, in turn, enhances reasoning remains insufficiently explored, lacking solid empirical evidence and community consensus.

\textbf{Benchmarking visual generation for reasoning.} This paper contributes to a growing line of research on visual generation for reasoning. RealUnify \cite{shi2025realunify} and Uni-MMMU \cite{zou2025uni} design tasks in which generation is expected to enhance reasoning, but report mixed results without revealing clear trends regarding the benefits of visual generation. ROVER \cite{liang2025rover} reveals fundamental limitations of current models in generating meaningful visual reasoning steps, often resulting in minimal or even negative gains in final accuracy. In contrast, MIRA \cite{zhou2025visualizing} conducts a sanity test by providing manually annotated visual cues, thereby bypassing the evaluation of visual world modeling capability. While the aforementioned works evaluate zero-shot performance, ThinkMorph \cite{gu2025thinkmorph} fine-tunes UMMs to reveal emergent reasoning behaviors but restricts each CoT to a single intermediate image, thereby not fully exploiting the potential of interleaved CoT. Our work distinguishes itself through a world-model perspective that enables a principled investigation, allowing us to both demonstrate and systematically explain when visual generation yields positive gains and when it does not.

\section{A World Model Perspective on Multimodal Reasoning}

Inspired by the aforementioned connections between human cognition and artificial intelligence, we formalize our world-model perspective on multimodal reasoning (see Figure~\ref{fig:formulation}) in this section.

\subsection{Formulation: Multiple Observations of the World}
\label{sec:momdp}

Without loss of generality, the world of a specific task can be formulated as a \textbf{multi-observable Markov decision process} (MOMDP) $\mathcal{M} = (\mathcal{S}, \mathcal{A}, p, \Phi, \mathcal{O}_{\phi}, e_{\phi})$, where $\mathcal{S}$ denotes the state space, $\mathcal{A}$ the action space, $p$ the transition function, $\Phi$ the parameter space of observation functions, $\mathcal{O}_{\phi}$ the observation space, and $e_{\phi}$ the observation function. Each $s \in \mathcal{S}$ represents the underlying state of the world, which is typically hidden and not directly observable. Instead, it can be perceived through different instantiations of observations (hereafter also referred to as \textit{views}) \cite{huh2024platonic}, given by $o = e_{\phi}(s) \in \mathcal{O}_{\phi}$, parameterized by $\phi \in \Phi$. As illustrated in Figure~\ref{fig:formulation}a, such views can span multiple modalities—for example, visual observations corresponding to different camera poses, or verbal descriptions expressed with different emphases or styles. When an action $a \in \mathcal{A}$ is applied to the current state, the world transits according to the dynamics $s^\prime \sim p(s^\prime|s, a)$ and yields new observations.

\subsection{Atomic Capabilities of World Models}

\begin{figure}[tbp]
  \centering
  \includegraphics[width=0.9\linewidth]{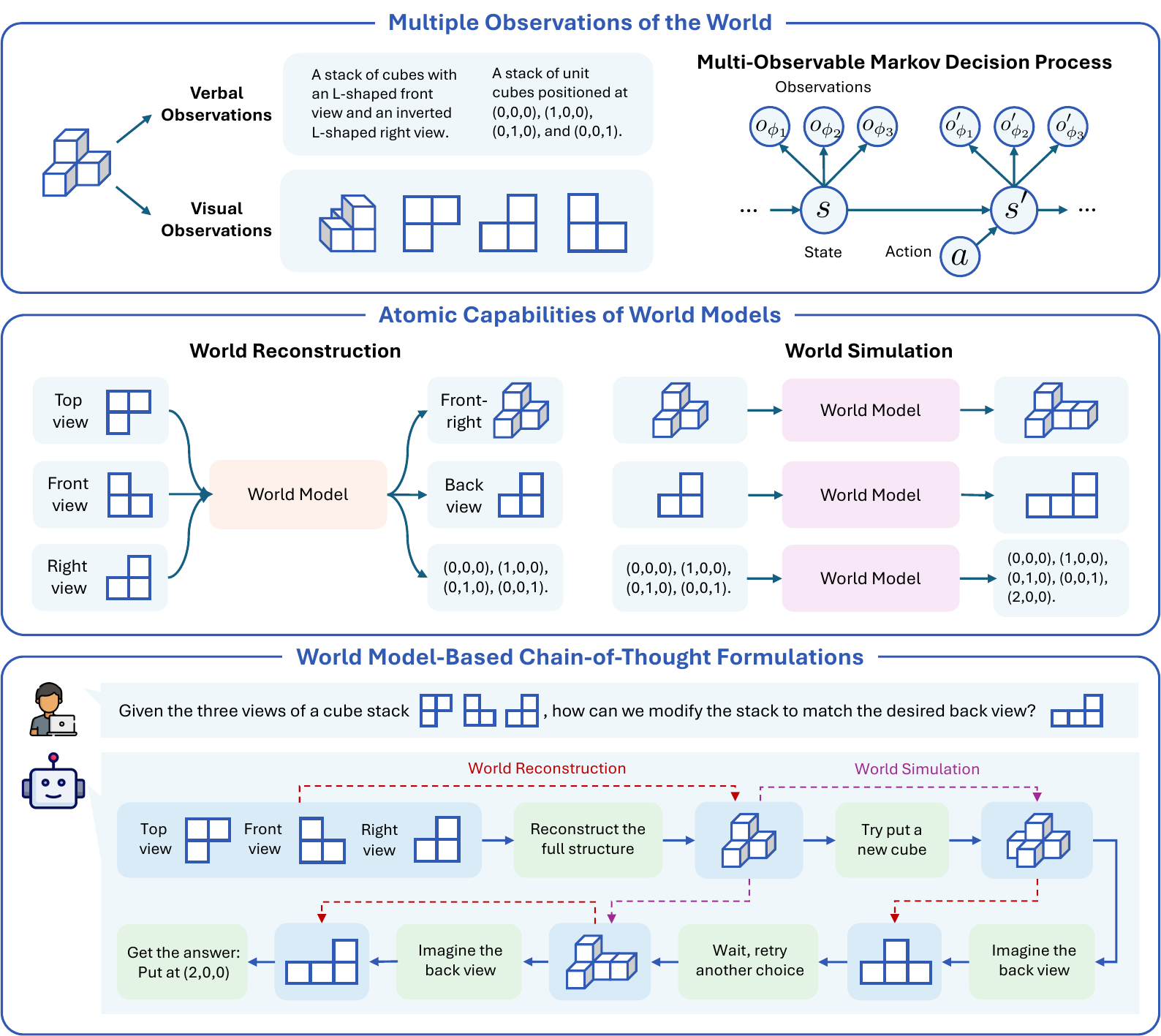}
  \caption{Theoretical formulation of the world model perspective on multimodal reasoning. (a) Observations of the same underlying world state can span multiple modalities, including verbal and visual observations, each reflecting different views or emphases. (b) Two atomic capabilities of world models are defined: \textit{world reconstruction}, which infers complete structure from partial observations and enables novel view synthesis, and \textit{world simulation}, which models dynamics to predict future observations. (c) Chain-of-thought reasoning includes internal world modeling, by explicitly maintaining an evolving sequence of observations, generated through either of the atomic world model capabilities.}
  \label{fig:formulation}
\end{figure}

A world model, analogous to human mental models, is then expected to support two fundamental capabilities \cite{lecun2022path}, illustrated in Figure~\ref{fig:formulation}b. The first is called \textbf{world reconstruction}. Humans are remarkably skilled at mentally reconstructing the structure of an environment from only a few partial observations \cite{yin2025spatial}, grounded in their prior knowledge of the world. Such mental reconstruction allows them to imagine novel views of the same underlying state, supporting skills such as mental rotation. Formally, the perception component of a world model encodes $n$ observations from limited views into an internal representation: $\hat s = \operatorname{enc}(o_{\phi_1}, \dots, o_{\phi_n}) \approx s$. This representation approximates the true state\footnote{We set aside the debate between compact and comprehensive representations. By treating abstract (e.g., sketches) and high-fidelity observations as different view specifications, this formulation allows the internal representation to flexibly adjust to the level of detail required by the desired views.}, and can then be decoded to generate an unseen observation: $\hat o_{\phi_{n+1}} = \mathrm{dec}(\hat s, \phi_{n+1}) \approx e_{\phi_{n+1}}(s)$, providing an internal "experience" of \textit{navigating} the world. In modern generative models, including UMMs, since their latent representations are not explicitly defined, the world reconstruction capability can be realized through end-to-end novel view generation: 
\begin{align}
\label{eq:world_recon}
p_\theta(o_{\phi_{n+1}}\mid o_{\phi_1}, \dots, o_{\phi_n}),
\end{align}
which implicitly learns the internal representations required to synthesize the new view. 

The second capability is \textbf{world simulation}. Humans can mentally simulate how the world evolves into the future, supporting reasoning and decision-making, either purely in their minds or with external aids such as a scratchpad. Formally, this corresponds to the prediction component of a world model, which predicts the transition of the current state and action: $\hat s^\prime \sim \operatorname{pred}(\hat s, a)$, providing an internal "experience" of \textit{interacting} with the world. Similarly, for modern generative models, this capability is more typically realized through predictions of future observations: \begin{align}
\label{eq:world_sim}
p_\theta(o_{t+1}\mid o_{\leq t}, a_{\leq t}).
\end{align}
In our new evaluation suite, we deliberately curate tasks that specifically demand each capability, allowing us to independently validate its contribution to multimodal reasoning (see Section~\ref{sec:benchmark}).

\subsection{Deliberate Reasoning with World Modeling Across Modalities}

We then formalize how world-modeling capabilities within multimodal models contribute to reasoning. Given a question $Q$ and input images $I$, the chain-of-thought reasoning process of a multimodal AI system can be expressed as a sequence of intermediate steps (or thoughts) $R = \tau_1, \tau_2, \dots, \tau_H$, followed by the answer $A$. Although this general formulation treats each reasoning step $\tau_i$ as an unconstrained, free-form operation, our world model perspective suggests that humans reason by prediction and planning, and each step inherently manipulates the underlying world observations of the problem \cite{wang2025vagen, copet2025cwm, zhang2025agent}. We therefore refine the reasoning formulation as $\tau_i = (r_i, o_i)$ to explicitly incorporate an evolving sequence of observations:
\begin{align}
    \label{eq:cot}
    R = \left(r_1, o_{1}\right), \left(r_2, o_{2}\right), \dots, \left(r_H, o_{H}\right),
\end{align}
where $r_i$\footnote{We use $i$ to index reasoning steps in order to distinguish them from the true time step $t$ of the underlying MOMDP. The twos are not generally aligned, as we may include branching and backtracking in the reasoning.} denotes a logical reasoning step based on the accumulated context, typically expressed in text, and $o_{i}$ denotes the observation generated at that step. Specifically, the input images serve as the initial observation $o_0 = I$, and subsequent observations are generated from previous reasoning and observations, by invoking atomic world modeling capabilities: world reconstruction (Eq.~\eqref{eq:world_recon}) and world simulation (Eq.~\eqref{eq:world_sim}), where reasoning steps imply actions $a$ and view transformations $\phi$, as illustrated in Figure~\ref{fig:formulation}c.

This formulation is modality-agnostic, allowing observations—and thus world modeling—to arise across various modalities. We focus specifically on verbal and visual observations, motivated by dual-coding theory in human cognition and by the fact that UMMs are equipped to generate both. This yields several concrete CoT instantiations. Specifically, \textbf{verbal world modeling} produces purely verbal CoTs, with $o_i$ as verbal descriptions, whereas \textbf{visual world modeling} produces verbal-visual interleaved CoTs, with $o_i$ as generated images. In addition, prior work has discovered that language models can implicitly learn world models with emergent internal representations of board-game states without explicit supervision \cite{li2023emergent}. Motivated by this, we also consider \textbf{implicit world modeling}, in which no explicit observation is generated ($o_i = \emptyset$)\footnote{In practice, strictly distinguishing implicit from verbal world modeling can be difficult, because there are often partial descriptions of the current state in the reasoning part $r_i$. In this work, we treat verbal world modeling as explicitly expressing world states or observations in text, such as coordinates or symbolic matrices.}.

\subsection{The Visual Superiority Hypothesis}

Contemporary LLMs and VLMs have achieved impressive performance in structured and abstract domains, such as mathematics and programming, largely driven by large-scale language-centric pre-training and verbal chain-of-thought post-training. Although these models have accumulated extensive verbal and symbolic knowledge, their understanding of the visual world remains limited when trained under purely verbal supervision. As a result, they continue to struggle with tasks grounded in basic physical and spatial intuition that even young children naturally master \cite{schulze2025visual, cai2025has}. 

Visual world modeling is therefore essential for endowing multimodal AI with complementary forms of information and knowledge. 
(1) In terms of \textit{informativeness}, while verbal and symbolic representations capture high-level semantic abstractions, they often suffer from ambiguity and representational bottlenecks. In contrast, visual observations are more concrete and information-rich, directly encoding physical properties such as motion and spatial relationships. This provides precise, fine-grained grounding for reasoning about the complex real world, particularly in spatial and physical tasks. 
(2) In terms of \textit{prior knowledge}, visual world knowledge is inherently complementary to symbolic knowledge. Humans and animals acquire much of this knowledge (e.g., physical interactions and spatial transformations) through perception, largely independent of language. Consequently, humans naturally represent and communicate such knowledge visually—for example, by sketching an approximate parabolic trajectory without performing explicit calculations. This suggests that different aspects of world knowledge are concentrated in different data modalities, and learning from large-scale generative modeling of visual data can thereby expand the effective knowledge landscape available for multimodal reasoning.

We next formalize and justify these insights through theoretical analysis, with formal statements and proofs provided in Appendix~\ref{app:theorectical}.

\textbf{Informativeness.} For notational convenience, we denote the question $Q$ as $r_0$, the input images as $o_0$, and the final answer as $r_{H+1}$. Prefixes of a CoT are defined as $R_i = (r_0, o_0, r_1, o_1, \dots, r_{i-1}, o_{i-1}), \tilde R_i = (r_0, o_0, r_1, o_1, \dots, r_{i-1}, o_{i-1}, r_i)$. We use $\mathbb{H}(\cdot)$ and $\mathbb{I}(\cdot;\cdot)$ to denote Shannon entropy and mutual information, respectively. We first establish that the end-to-end answer error admits an upper bound that naturally decomposes into reasoning and world-modeling errors.

\begin{thm}
\label{thm:kl_chain_rule}
Let $p$ denote the distribution over optimal chain-of-thoughts and answers, and let $p_\theta$ be a learned reasoning model. Then the following inequality holds:
\begin{align}
\nonumber
\operatorname{KL}(p(A \mid Q, I)\mid\mid p_\theta(A \mid Q, I)) &\leq \operatorname{KL}(p(R, A \mid Q, I)\mid\mid p_\theta(R, A \mid Q, I))\\
= &\sum_{i=1}^{H+1} \underbrace{\mathbb{E}_p\left[\operatorname{KL}(p(r_i| R_i) \mid\mid  p_\theta(r_i| R_i) )\right]}_{\textnormal{reasoning errors}} + \sum_{i=1}^{H}  \underbrace{\mathbb{E}_p\left[\operatorname{KL}(p(o_i| \tilde R_i) \mid\mid p_\theta(o_i| \tilde R_i))\right]}_{\textnormal{world-modeling errors}}.
\label{eq:decompose}
\end{align}
\end{thm}

This decomposition reveals a fundamental trade-off between the informativeness of world models for reasoning and the fidelity of the world model itself. In the case of implicit world modeling, where $o_i = \emptyset$, we get rid of the world-modeling error. However, this typically comes at the cost of increased uncertainty and learning difficulty in reasoning, as all state transitions must be implicitly encoded. Empirically, world models that explicitly track the task states, serving as verbal or visual sketchpads, are generally beneficial for reasoning. We dive into the reasoning component of Eq.~\eqref{eq:decompose} to elucidate the factors underlying these benefits.

\begin{thm}
\label{thm:mutual_info}
Let $s_i$ denote the latent states associated with the observations $o_i$. Under appropriate assumptions, the reduction in reasoning uncertainty achieved by explicit world modeling satisfies the following properties:
\begin{enumerate}
    \item Reasoning uncertainty does not increase: $\mathbb{H}(r_i| o_{0}, r_{0:i-1}) - \mathbb{H}(r_i| R_i) = \mathbb{I}(o_{1:i-1}; r_i| o_0,r_{0:i-1}) \geq 0.$
    \item The reasoning uncertainty improvement is bounded by both \textnormal{(i)} the information that observations provide about the underlying states and \textnormal{(ii)} the information that the reasoning step requires about those states:
    \begin{equation}
    \mathbb{I}(o_{1:i-1}; r_i| o_0,r_{0:i-1}) \leq \min\left(\mathbb{I}(o_{1:i-1}; s_{1:i-1}), \mathbb{I}(r_i; s_{0:i-1}, r_{0:i-1})\right).
    \label{eq:mi_ub}
    \end{equation}
\end{enumerate}
\end{thm}

The uncertainty of the target distribution is closely related to sample efficiency and learning difficulty. Consequently, the upper bound on the improvement of reasoning uncertainty (Eq.~\eqref{eq:mi_ub}) highlights another trade-off in the choice of observation modality for world modeling. The first term indicates that observations should be sufficiently informative about the underlying latent states. In contrast, the second suggests that they need only preserve the task-relevant aspects of the states required to select appropriate reasoning steps. Excessively detailed observations may be unnecessary and even detrimental, increasing world modeling errors.

\textbf{Prior knowledge.} Although visual world models are more informative, they are intrinsically more difficult to learn from scratch due to the high dimensionality and complexity of visual observations. Fortunately, modern AI systems are typically large-scale pre-trained, which endows them with strong prior knowledge and enables faster convergence and improved generalization during downstream post-training. As discussed earlier, humans tend to represent different aspects of world knowledge through different modalities. Consequently, for a given downstream task, the distribution shift between its transition distribution and that learned during large-scale Internet pre-training can vary substantially across modalities. The generalization bound in Theorem~\ref{thm:finetune_nontrivial} of Appendix~\ref{app:theoretical_transfer} suggests that this modality-dependent distribution shift is closely related to the post-training sample efficiency of the corresponding world model. This highlights the importance of acquiring broad prior knowledge across modalities during pre-training, and of leveraging the proper modality whose priors are best aligned with the downstream task.

Drawing on the above analysis, we formulate our central hypothesis regarding when and how visual generation benefits reasoning, thereby helping narrow the gap between multimodal AI and human capabilities.

\begin{tcolorbox}[colback=blue!2!white,leftrule=2.5mm,size=title]
\textbf{The Visual Superiority Hypothesis:} In multimodal reasoning tasks grounded in the physical world, visual generation as a world model yields representations that are more informative and knowledge-rich than those produced by verbal world models.
\end{tcolorbox}

\section{Experiment Settings}
\label{sec:settings}

Finally, we empirically validate the insights and theoretical analyses presented above through a series of controlled experiments. In this section, we describe the evaluation tasks and model training procedures.

\subsection{VisWorld-Eval: Task Suite for Reasoning with Visual World Modeling}
\label{sec:benchmark}

While prior work has primarily designed evaluation tasks heuristically, we principledly evaluate multimodal reasoning across tasks designed to specific world model capabilities. Building on related benchmarks, we identify and curate a total of seven tasks, forming an evaluation suite tailored to assess reasoning with visual world modeling. All tasks are framed as question answering with concise, verifiable answers, and performance is measured by answer accuracy. We refer to this suite as \textit{VisWorld-Eval}, and summarize it in Figure~\ref{fig:benchmark}.

\textbf{World simulation.} We consider the following tasks that primarily require simulating world dynamics over time: (1) \textit{Paper folding}: Adapted from SpatialViz-Bench \cite{wang2025spatialviz}, this task presents a sequence of paper folds followed by hole punching, and asks for the distribution of holes after the paper is unfolded. Successfully solving this task requires simulating the unfolding process, relying on prior knowledge of symmetry and spatial transformations that is commonly grounded in visual experience. (2) \textit{Multi-hop manipulation}: Build upon CLEVR \cite{johnson2017clevr}, this task features a scene containing objects with various shapes and colors that undergo a sequence of operations, such as addition, removal, or color changes. The final question queries properties of the resulting layouts. Since target objects of operations are often specified via relative spatial relationships, this task places strong demands on state tracking and spatial understanding. (3) \textit{Ball tracking}: Adapted from RBench-V \cite{guo2025rbench}, this task evaluates physical dynamics simulation by requiring the model to infer the trajectory of a ball undergoing ideal specular reflections within a given scene and predicting which numbered hole it will ultimately enter. In addition, we include (4) \textit{Maze} \cite{maze-dataset} and (5) \textit{Sokoban} \cite{tong2025gamerlsynthesizingmultimodalverifiable}, as these two grid-world tasks are commonly used in prior work of studying visual generation for reasoning \cite{xu2025visual, li2025imagine}.

\textbf{World reconstruction.} We also evaluate tasks that emphasize reconstructing underlying world structure from partial observations: (6) \textit{Cube 3-view projection}: Adapted from SpatialViz-Bench \cite{wang2025spatialviz}, this task provides an isometric view and two orthographic views of a connected cube stack, and asks about an unseen viewpoint. Solving the task requires reconstructing the full 3D structure and mentally rotating or projecting it into the queried view, a process closely aligned with human visual mental representations. (7) \textit{Real-world spatial reasoning}: We focus on the positional relationship subset of MMSI-Bench \cite{yang2025mmsi}. Given multiple views of a realistic scene, these tasks ask about positional relationships among the cameras, objects, and regions. Successfully answering these questions requires constructing a coherent spatial mental model of the scene from limited viewpoints to support accurate spatial reasoning.

For each task, we construct SFT data by designing different CoT patterns with implicit, verbal, or visual world modeling, enabling controlled comparative evaluations. Data construction pipeline and examples across tasks are presented in Appendix~\ref{app:task_details}.

\begin{figure}[tbp]
    \centering
    \includegraphics[width=\linewidth]{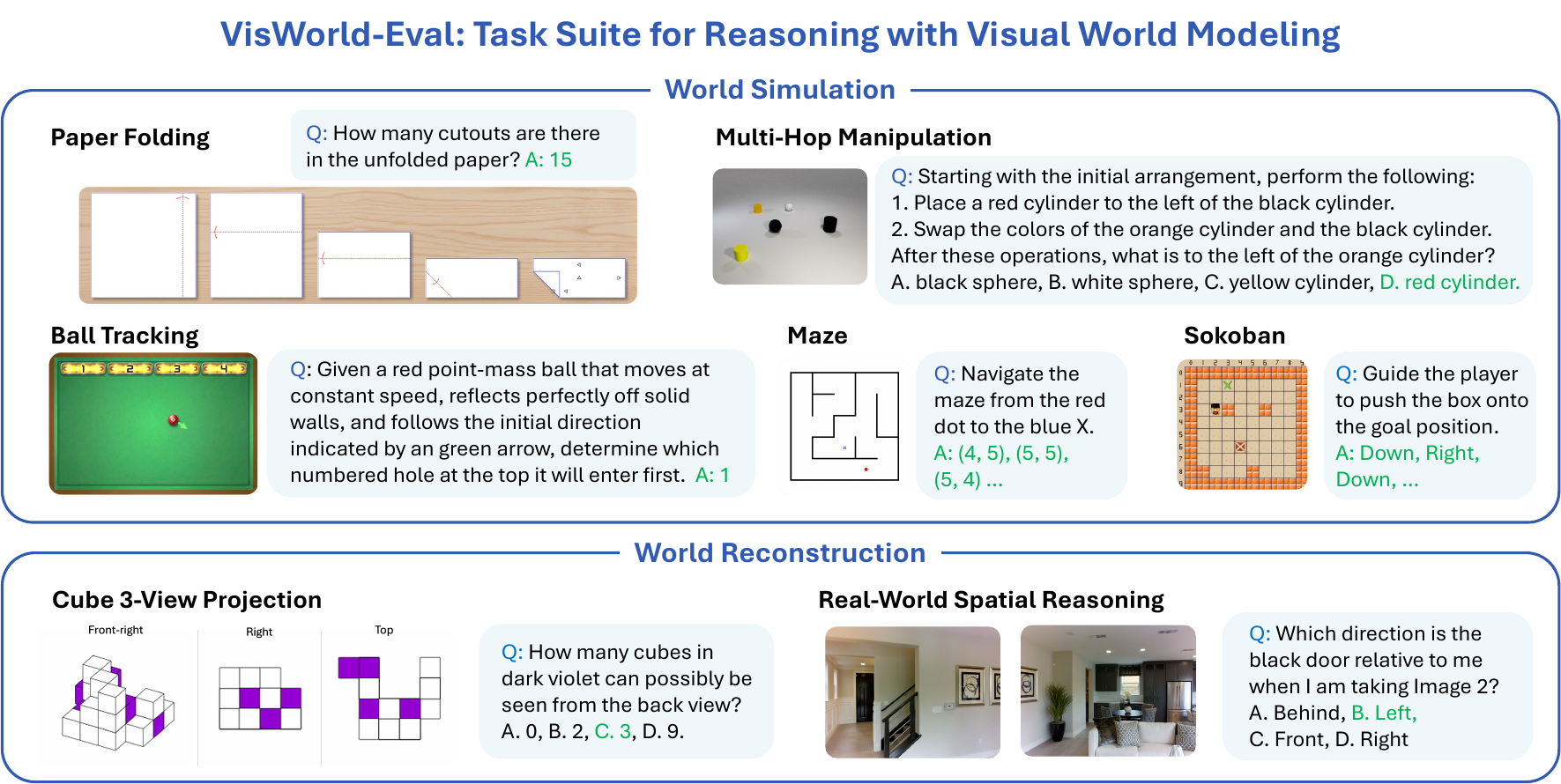}
    \caption{The VisWorld-Eval suite for assessing multimodal reasoning with visual world modeling. VisWorld-Eval comprises seven tasks spanning both synthetic and real-world domains, each designed to isolate and demand specific atomic world-model capabilities.}
    \label{fig:benchmark}
\end{figure}

\begin{table}[tbp]
\caption{Zero-shot evaluation of advanced VLMs on VisWorld-Eval. We report the average accuracy over five tasks (excluding Maze and Sokoban) and over all seven tasks.}
\label{tab:leaderboard}
\centering
\setlength{\tabcolsep}{5pt}
\small
\begin{tabular}{l|ccccccc|cc}
\toprule
Models      & 
\begin{tabular}[c]{@{}c@{}}\scalebox{0.9}{Paper}\\\scalebox{0.9}{Folding}\end{tabular} & 
\begin{tabular}[c]{@{}c@{}}\scalebox{0.9}{Multi-Hop}\\\scalebox{0.9}{Manip.}\end{tabular} & 
\begin{tabular}[c]{@{}c@{}}\scalebox{0.9}{Ball}\\\scalebox{0.9}{Tracking}\end{tabular} & 
\begin{tabular}[c]{@{}c@{}}\scalebox{0.9}{Cube}\\\scalebox{0.9}{3-View}\end{tabular} & 
\begin{tabular}[c]{@{}c@{}}\scalebox{0.9}{MMSI}\\\scalebox{0.9}{(Pos. Rel.)}\end{tabular} &  
\scalebox{0.9}{Maze} & 
\scalebox{0.9}{Sokoban} & 
\begin{tabular}[c]{@{}c@{}}Overall\\\scalebox{0.9}{(5 tasks)}\end{tabular} & 
\begin{tabular}[c]{@{}c@{}}Overall\\\scalebox{0.9}{(7 tasks)}\end{tabular} 
\\ \midrule
\multicolumn{10}{c}{\textit{Proprietary Models}} \\ \midrule
Gemini 3 Flash & \underline{25.6}  &  \textbf{75.4}   &  \textbf{55.3}   & \underline{52.7}   &  41.3     & \underline{73.9}    &   \textbf{99.3}  &  \textbf{50.0} & \textbf{60.5}     \\
Gemini 3 Pro &  \textbf{27.0} &  74.5   &  \underline{44.7}   &  \textbf{53.3}  &  \textbf{49.6}  & 33.5    &  \underline{90.2}   &  \underline{49.8} &  \underline{53.2}      \\
Seed 1.8  & 10.6  & \underline{75.2}   &  24.4   & 42.5   &  38.8    & \textbf{83.9}    &  68.3   &   38.3 &  49.1   \\
GPT 5.1  & 6.4  & 73.9   & 34.8   & 44.5 & \underline{44.8} &   0.6  &  62.8 & 40.8  &  38.2     \\
o3 & 13.5  &  68.1   &  24.7   & 37.7   &  44.4    &   0.0  &   36.0  &  37.6 &  32.0    \\
\midrule
\multicolumn{10}{c}{\textit{Open-Source Models}} \\ \midrule
\scalebox{0.9}{Qwen3-VL-8B-Thinking} \cite{bai2025qwen3vltechnicalreport}  & 11.0  & 49.3    &  17.8   &  21.2  &  27.7    &   0.0  &  5.8   &   25.4 & 18.9       \\
\scalebox{0.9}{BAGEL-7B-MoT \cite{deng2025emerging}}  & 11.2  & 31.6    & 19.4    &  26.8  & 27.2     & 0.0    &  0.2   &  23.2 & 16.6       \\
\bottomrule
\end{tabular}
\end{table}

\textbf{Evaluation of advanced VLMs.} Table~\ref{tab:leaderboard} reports the zero-shot performance of advanced VLMs on VisWorld-Eval. Overall, these models perform suboptimally, highlighting limitations of current multimodal AI systems. Among them, Gemini 3 Flash and Gemini 3 Pro remarkably outperform the other models; however, their performance remains far from satisfactory on challenging tasks like paper folding, ball tracking, cube 3-view projection, and real-world spatial reasoning.

\subsection{Unified Multimodal Model Training and Evaluation}

\textbf{Evaluation protocol.} To investigate the benefits of visual generation in multimodal reasoning, we evaluate post-trained UMMs, rather than the zero-shot performance of base models. To the best of our knowledge, no open-source model has been natively optimized for interleaved verbal-visual generation for reasoning. Even commercial closed-source models currently exhibit fundamental limitations in generating visual intermediate reasoning steps \cite{liang2025rover, zhou2025visualizing}. Focusing on post-trained models, therefore, provides a more meaningful estimate of the upper bound for multimodal reasoning performance, while reducing confounding effects arising from insufficient pre-training due to limited interleaved data availability or quality.

\textbf{Model training.} We adopt BAGEL \cite{deng2025emerging}, a state-of-the-art open-source unified multimodal model, as our base model. Most experiments are conducted by supervised fine-tuning (SFT) on task-specific datasets, where verbal and visual generation in both chain-of-thought reasoning and final answers are optimized using cross-entropy and flow-matching loss. Specifically, the loss for reasoning with visual world modeling is as follows:
\begin{equation}
    \mathcal{L}_\theta(Q, I, R, A) = -\sum_{i=1}^{H+1} \sum_{j=1}^{|r_i|}\log p_\theta\left(r_{i,j}\mid r_{i, <j}, R_i\right) + \sum_{i=1}^H \mathbb{E}_{t, \epsilon}\left\|v_\theta(o_i^t, t\mid \tilde R_i) - (\epsilon-o_i)\right\|_2^2,
\end{equation}
where $o_i^t = to_i + (1-t)\epsilon$ are noisy observations. We emphasize that in our formulation, $r_i$ refers to a verbal reasoning step, instead of a reward.
We also perform reinforcement learning from verifiable rewards (RLVR) following SFT. During RL, only the verbal generation component is optimized by GRPO \cite{guo2025deepseek}, while visual generation is regularized via the KL-divergence with respect to the SFT-trained reference model:
\begin{align}
\nonumber
    \mathcal{J}_\theta(Q, I) = \mathbb{E}_{o, r\sim p_{\theta_\text{old}}}
\Bigg[ &\sum_{i=1}^{H+1} \sum_{j=1}^{|r_i|} \Bigg(\min \Big( \frac{p_\theta\left(r_{i,j}\mid r_{i, <j}, R_i\right)}{p_{\theta_{\text{old}}}\left(r_{i,j}\mid r_{i, <j}, R_i\right)} {A},  
\ \text{clip} \Big( \frac{p_{\theta}\left(r_{i,j}\mid r_{i, <j}, R_i\right)}{p_{\theta_{\text{old}}}\left(r_{i,j}\mid r_{i, <j}, R_i\right)}, 1 - \varepsilon, 1 + \varepsilon \Big) {A} \Big) \Bigg) \\
&-\sum_{i=1}^H \mathbb{E}_{t, \epsilon}\left\|v_\theta(o_i^t, t\mid \tilde R_i) - v_{\theta_{\text{ref}}}(o_i^t, t\mid \tilde R_i)\right\|_2^2\Bigg].
\end{align}

Full implementation details and hyperparameters are provided in Appendix~\ref{app:training}.

\section{Experimental Results}

In this section, we demonstrate that visual world modeling boosts multimodal reasoning through two atomic capabilities: world simulation (Section~\ref{sec:exp_simu}) and world reconstruction (Section~\ref{sec:exp_recon}). We also identify tasks in which it is unhelpful (Section~\ref{sec:exp_negative}), where implicit or verbal world modeling is sufficient. We conduct analysis in detail. Interestingly, we reveal emergent internal representations in UMMs that support implicit world modeling on simple maze tasks.

\subsection{Visual World Simulation Boosts Multimodal Reasoning}
\label{sec:exp_simu}

\textbf{Main results.} Figure~\ref{fig:overall_sft} summarizes the performance of SFT-trained UMMs under different chain-of-thought formulations across all tasks. We observe that interleaved CoT with visual world modeling significantly outperforms its purely verbal counterparts on three world simulation tasks: paper folding, multi-hop manipulation, and ball tracking. These gains are attributed to both the richer expressiveness and stronger prior knowledge afforded by the visual modality. In particular, it is difficult for models to precisely ground object coordinates and perform arithmetic operations without external tools in tasks such as multi-hop manipulation and ball tracking, with the latter being especially challenging. Thus, verbal world modeling is inappropriate and omitted in these tasks. This exacerbates ambiguity and hallucinations in purely verbal reasoning. Similarly, in paper folding, although models can track the states of holes, it remains difficult to completely depict the paper contour during unfolding. Moreover, as showcased in Figure~\ref{fig:showcase} and~\ref{fig:paperfolding_failure_case}, the spatial transformation involved in paper unfolding critically relies on an understanding of geometric symmetry, which can be more naturally learned from visual data like images and videos.

\textbf{Sample efficiency.} To further demonstrate the stronger prior knowledge embedded in the visual modality, we experiment comparing the sample efficiency of verbal and visual world modeling on the paper folding task. As shown in Figure~\ref{fig:sample_eff}, reasoning with visual world modeling exhibits substantially higher sample efficiency, achieving performance comparable to verbal world modeling while using more than $4\times$ less SFT data.

\subsection{Visual World Reconstruction Boosts Multimodal Reasoning}
\label{sec:exp_recon}

\textbf{Main results.} As shown in Figure~\ref{fig:overall_sft}, multimodal reasoning tasks that rely on world reconstruction capabilities also benefit substantially from visual world modeling. In the cube 3-view task, predicting a novel view of stacked cubes, denoted symbolic character matrices, suffers from limited prior knowledge, whereas visually rotating objects has been a rich experience during pre-training with large-scale Internet videos. For MMSI tasks, fully describing a novel view of a realistic scene using text alone is similarly ill-suited as in the previous subsection, and we also discover hallucinations in pure verbal reasoning, which lacks grounding to visual generation. We do not observe consistent improvements on other positional-relationship subtasks in MMSI-Bench, except camera-object and camera-region, which we attribute to current UMM's limitations in both spatial understanding during verbal reasoning and generation quality in visual world modeling. Full quantitative results and qualitative examples are provided in Appendix~\ref{app:exp}. We expect these limitations to be mitigated in future work with stronger base models.

\begin{figure}[tbp]
    \centering
    \includegraphics[width=\linewidth]{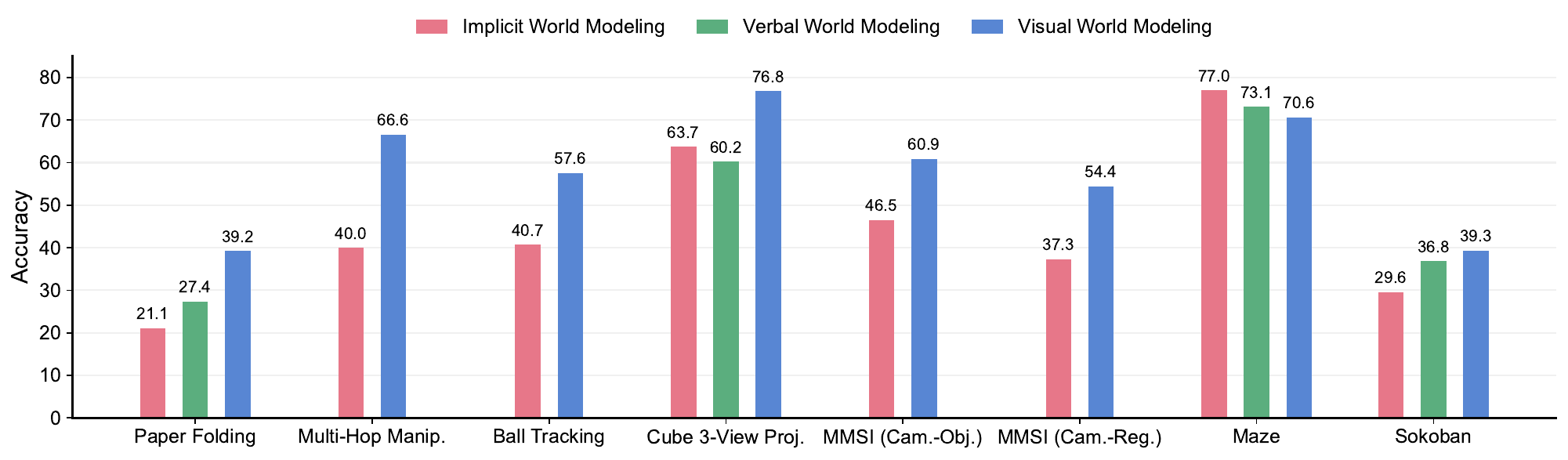}
    \caption{Performance of SFT-trained UMMs with different world model-based chain-of-thought formulations across seven tasks from VisWorld-Eval. Refer to Table~\ref{tab:leaderboard} for zero-shot performance of advanced VLMs.}
    \label{fig:overall_sft}
\end{figure}

\textbf{Effects of task difficulties.} Figure~\ref{fig:cube_analysis} analyzes performance on the cube 3-view projection task across varying sizes of input cube stacks. We observe a consistent advantage of reasoning with visual world modeling over verbal world modeling across all difficulty levels. Notably, for cube stacks of size six—out of the training distribution—visual world modeling still yields approximately a $10\%$ performance improvement.

\textbf{World model fidelity.} Modern AI models are known to exhibit hallucinations along their reasoning trajectories, even when producing correct final answers \cite{liang2025rover}. We therefore evaluate the fidelity of world modeling in the cube 3-view projection task by comparing ground-truth views with the intermediate views generated verbally or visually during reasoning. To focus on structural correctness, we compare only the shapes of the views and completely ignore color information. Even under this relaxed evaluation setting, Figure~\ref{fig:cube_analysis} shows that verbal world modeling exhibits dramatically low fidelity, with scores degrading to near zero. Notably, approximately half of the samples require predicting the opposite view of a given input view, a transformation that only involves horizontal mirroring. Visual world modeling, benefiting from stronger prior knowledge of such geometric transformations, captures these patterns effectively and achieves fidelity scores consistently exceeding $50\%$.

\subsection{Visual World Modeling is Unhelpful for Certain Tasks}
\label{sec:exp_negative}

\textbf{Main results.} (Un)surprisingly, we do not observe notable improvements on grid-world tasks, including maze and Sokoban. In the maze tasks, reasoning with implicit world modeling—without explicitly tracking coordinates—achieves the best performance with a slight advantage. These results are consistent with recent empirical findings \cite{du2025revisiting}. We argue that this is also well explained by our world model perspective. In these tasks, state tracking is relatively simple, typically requiring the maintenance of only one or two two-dimensional coordinates, which can be adequately handled through verbal reasoning alone. Furthermore, in the maze task, we hypothesize that such world modeling can be implicitly encoded in the model's hidden representations \cite{li2023emergent}, which helps explain the competitive performance of verbal reasoning without explicit coordinate tracking.

\begin{wrapfigure}{r}{0.39\textwidth}
  \setlength{\columnsep}{3pt}
  \vspace{-8pt}
  \begin{center}
    \includegraphics[width=0.96\linewidth]{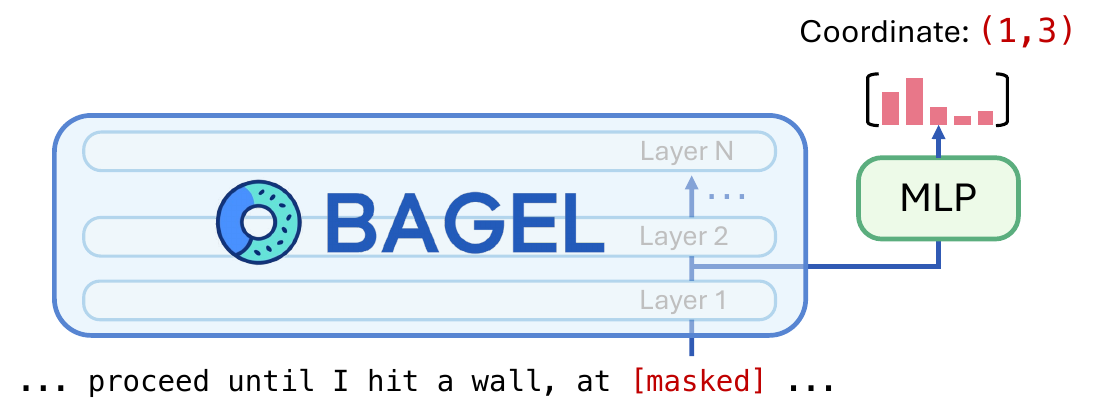}
  \end{center}
  \caption{Probing implicit world models, by training a set of probes, i.e., MLPs which infer the masked point coordinates during reasoning from internal representations.}
  \label{fig:implicit_wm_prob}
  \vspace{-10pt}
\end{wrapfigure}

\textbf{Demystifying implicit world modeling.} To validate this hypothesis, we probe the internal representations of models, as illustrated in Figure~\ref{fig:implicit_wm_prob}. We consider the same architecture, BAGEL, with three different sets of weights: a randomly initialized model, the pre-trained model, and the model supervised fine-tuned on CoT data in the implicit world modeling format, in which special tokens mask all explicit point coordinates during the reasoning process. For each model, we extract the hidden representations of these special tokens at each layer. We then train multilayer perceptrons (MLPs) on these representations to predict the underlying true point coordinates. 

Figure~\ref{fig:implict_wm} reports the prediction accuracy on a validation set. As expected, the randomly initialized model completely fails to internally track point states, achieving only random-guess accuracy on $5\times 5$ mazes. In contrast, the pre-trained model \cite{deng2025emerging} already exhibits emergent representations that are predictive of maze states. Notably, we observe a non-monotonic trend across layers: prediction accuracy increases from lower layers (which capture low-level features) to middle layers, and then decreases toward the final layers, which are likely specialized for next-token prediction. Finally, supervised fine-tuning on domain-specific data, despite providing no explicit coordinate supervision, substantially enhances this internal predictability, achieving near-perfect accuracy. These in-depth results help explain our main experimental findings: as the model already possesses the capability for implicit world modeling, it does not necessarily benefit from explicit verbal world modeling, let alone more complex forms of visual world modeling.

\begin{figure}[tbp]
    \centering
    \begin{minipage}[!ht]{0.3\textwidth}
        \centering
        \includegraphics[width=\textwidth]{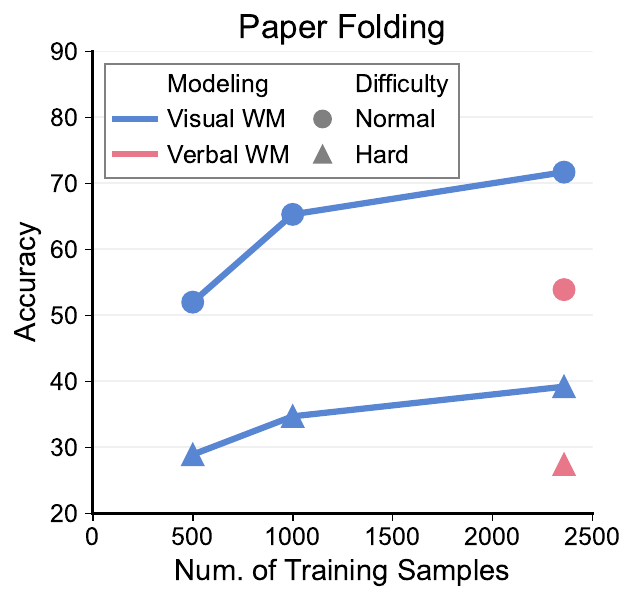}
        \subcaption{Sample efficiency.}
        \label{fig:sample_eff}
    \end{minipage}
    \begin{minipage}[!ht]{0.3\textwidth}
        \centering
        \includegraphics[width=\textwidth]{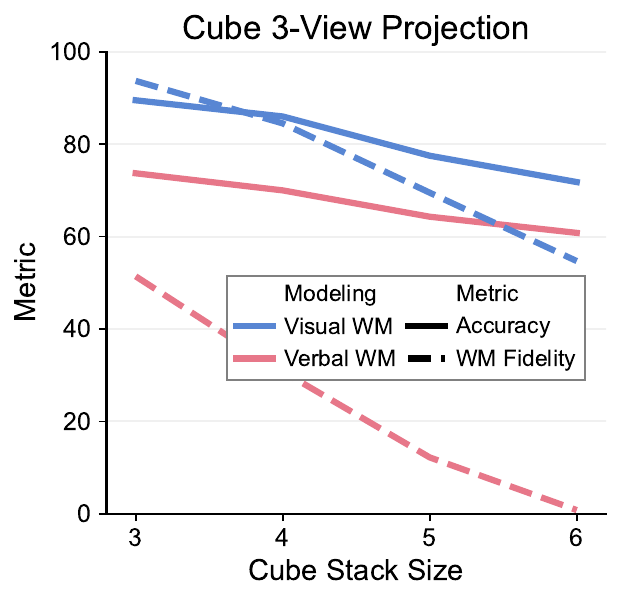}
        \subcaption{World model fidelity.}
        \label{fig:cube_analysis}
    \end{minipage}
    \begin{minipage}[!ht]{0.3\textwidth}
        \centering
        \includegraphics[width=\textwidth]{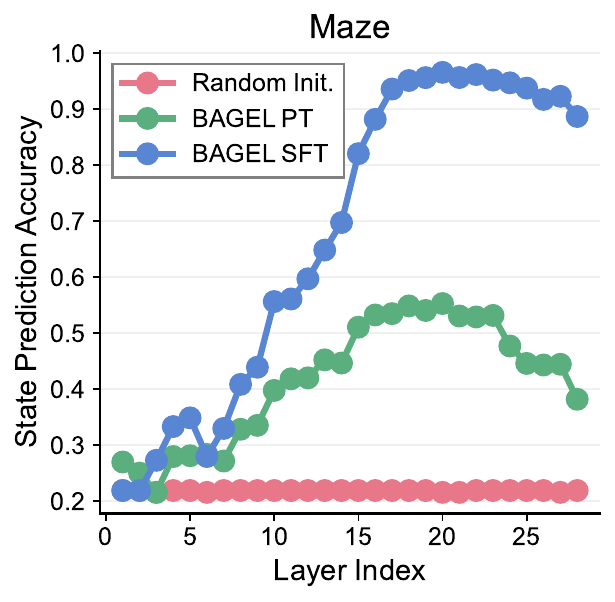}
        \subcaption{Implicit world modeling.}
        \label{fig:implict_wm}
    \end{minipage}
    \caption{Model analysis: (a) Performance of UMMs on the paper-folding task with varying numbers of SFT samples. Reasoning with visual world modeling achieves a $4\times$ improvement in sample efficiency. WM = world modeling. (b) Performance of UMMs on the cube 3-view projection task with increasing sizes of input cube stacks, evaluated using both answer accuracy and world-model fidelity. Visual world modeling demonstrates dramatically better fidelity of view synthesis. (c) Prediction accuracy of masked point coordinates in CoTs using representations extracted from different layers of different UMMs, revealing emergent internal world representations. PT = Pre-trained.}
    \label{fig:analysis}
\end{figure}

\begin{figure}[tbp]
    \centering
    \includegraphics[width=0.85\linewidth]{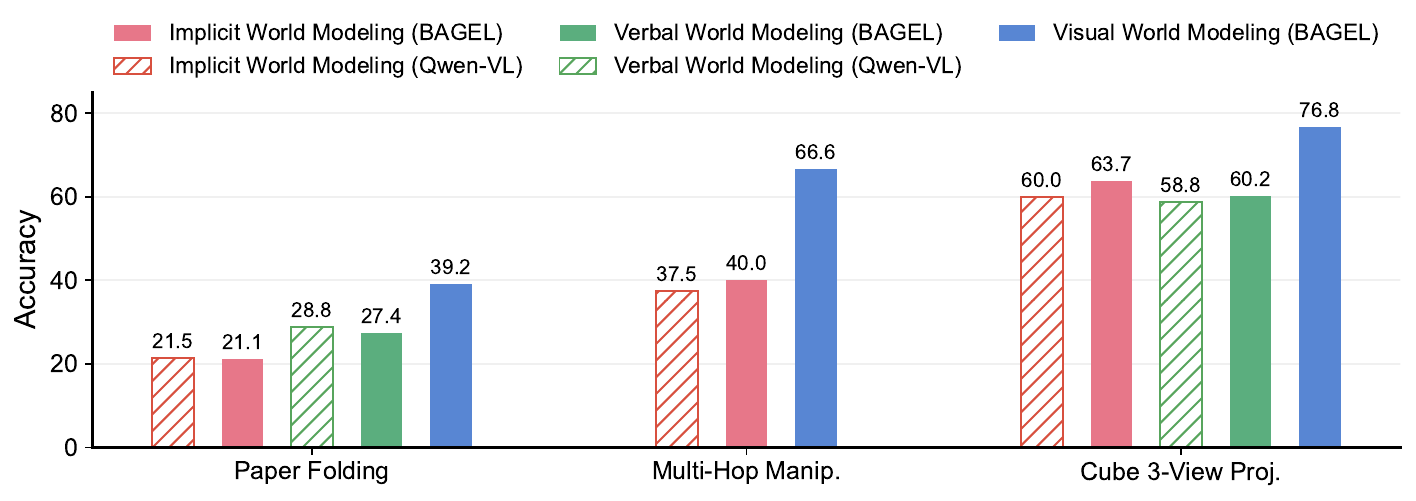}
    \caption{Performance of SFT-trained VLMs compared with UMMs across three tasks.}
    \label{fig:sft_vlm}
\end{figure}

\begin{figure}[tbp]
    \centering
    \includegraphics[width=0.85\linewidth]{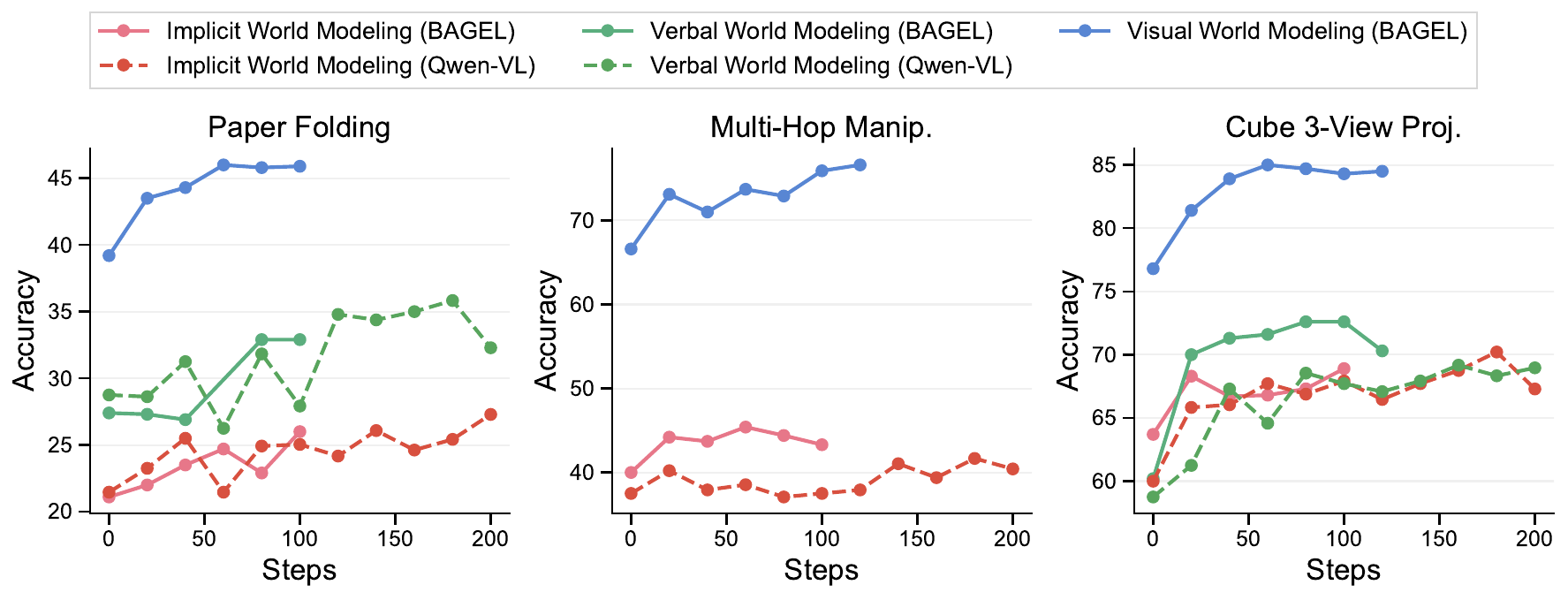}
    \caption{Performance of RLVR-trained VLMs and UMMs with different world-model-based CoT formulations across three tasks.}
    \label{fig:overall_rl}
\end{figure}

\subsection{Comparison with VLMs: Do UMMs Compromise Verbal Reasoning Capabilities?}
\label{sec:exp_vlm}

One may argue that UMMs are typically trained with a stronger emphasis on visual generation \cite{deng2025emerging}, which could compromise verbal reasoning capabilities, and bias comparisons in favor of visual world modeling. To address this concern, we compare with a pure VLM baseline, Qwen2.5-VL-7B-Instruct \cite{bai2025qwen2}, which shares the same Qwen 2.5 LLM base model, with BAGEL. We fine-tune Qwen2.5-VL on the same verbal CoT datasets used in the previous subsections and evaluate it on three representative tasks: paper folding, cube 3-view projection, and multi-hop manipulation.

\textbf{Results.} As shown in Figure~\ref{fig:sft_vlm}, the SFT performance of Qwen2.5-VL with implicit and verbal world modeling is comparable to that of BAGEL, without exhibiting significant advantages. It still lags behind BAGEL in settings that leverage visual world modeling. These results indicate that our findings arise from the inherent advantages of visual world modeling rather than from compromised verbal reasoning capabilities in UMMs.

\subsection{RL Enhances Various CoTs, Yet Does Not Close the Gap}
\label{sec:exp_rl}

Reinforcement learning from verifiable rewards (RLVR) has been a major driver of recent progress in reasoning models equipped with verbal chain-of-thoughts, achieving strong performance across domains such as mathematics \cite{guo2025deepseek}. While Figure~\ref{fig:overall_sft} shows a clear advantage of reasoning with visual world modeling after SFT, RLVR may further incentivize emergent reasoning behaviors that improve verbal CoTs. We thus conduct comparative RLVR experiments across different world model–based CoT formulations on three representative tasks.

\textbf{Results.} Figure~\ref{fig:overall_rl} presents the learning curves under RLVR for different models. We observe consistent improvements during RLVR for different CoT formulations. However, the performance gap persists. We also find that VLMs and UMMs generally perform similarly with verbal CoTs. These results suggest that the superiority arises from inherent advantages of the world modeling approach, rather than insufficient post-training. Notably, RL enhances reasoning with visual world modeling, even though only the verbal generation components of interleaved CoTs are directly optimized. We envision that the full potential of interleaved CoTs will be further released with the development of RL algorithms tailored for verbal-visual interleaved generation.

\section{Discussions}
\label{sec:discussions}

By bridging concepts from human cognition and artificial intelligence, we revisit the mechanisms underlying human reasoning and the central role that world models play. This provides a new perspective on the use of visual generation for reasoning in multimodal AI, highlighting its potential to serve as visual world models that complements the verbal world models embedded in LLMs, thereby enabling more human-like reasoning on scenarios grounded in the physical world. For the first time, this perspective is studied in a principled manner, through theoretical formulations that bridge world models and reasoning, as well as through empirical evaluations whose results are well explained by and strongly support the proposed insights. We hope this work helps address longstanding questions about the synergistic effects between generation and reasoning, and more broadly contributes to the development of more human-like AI that thinks and acts with multimodal world models.

\begin{figure}[tbp]
    \centering
    \includegraphics[width=\linewidth]{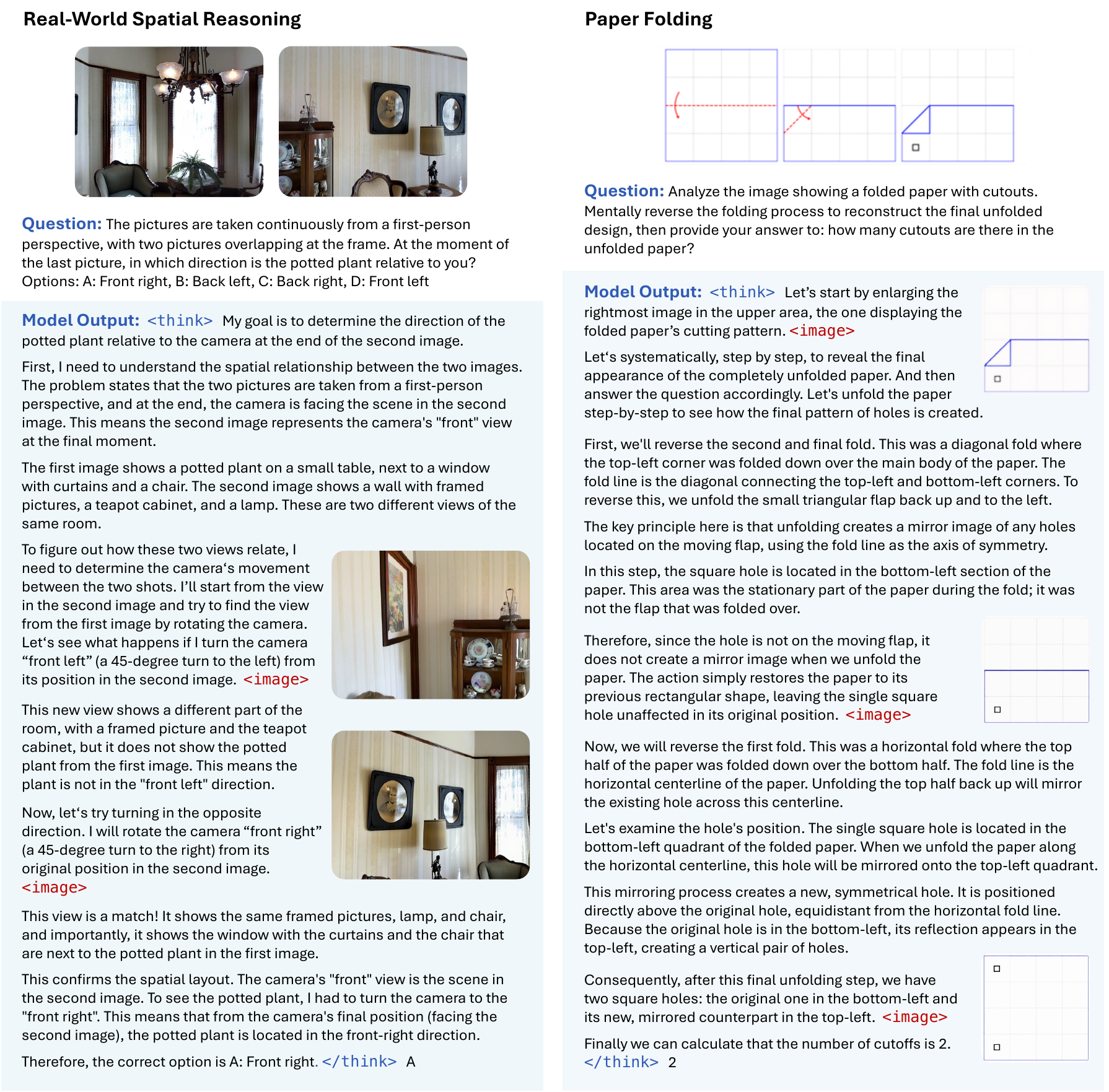}
    \caption{Showcases of interleaved verbal-visual chain-of-thought reasoning, generated by post-trained UMMs, where visual generation serves as world models. \texttt{<image>} denotes a placeholder indicating the position of a generated image.}
    \label{fig:showcase}
\end{figure}

\textbf{Limitations and future work.} This work primarily focuses on spatial and physical reasoning tasks, where multimodal AI systems exhibit a pronounced performance gap relative to humans. Many other tasks proposed in the related literature can also be interpreted through our world model perspective. For example, a prominent class of benchmarks involves visual jigsaw tasks \cite{shi2025realunify, gu2025thinkmorph, liang2025rover, zou2025uni}, in which input image patches are cropped, masked, or shuffled. Such tasks essentially probe a form of world reconstruction capability, as corrupted images and videos are commonly treated as specific views within the world model literature \cite{assran2023self, bardes2024revisiting, assran2025v}. Another active area of interest lies in STEM reasoning. Recent work \cite{shi2025mathcanvas} leverages visual generation for mathematical diagram editing, such as constructing auxiliary geometric lines. This closely resembles how humans use visual sketchpads to support math understanding and reasoning, constructing visual world models of a symbolic system. However, as symbolic representations in mathematics are largely complete, and mathematical reasoning has been extensively optimized in modern LLMs, it remains unclear whether multimodal interleaved CoT can fundamentally break through the performance limit, warranting further investigation.

We do not apply reinforcement learning to the visual generation components of verbal–visual interleaved CoTs \cite{liu2025flow}. Prior work has shown that world models themselves can be improved through RLVR \cite{wu2025rlvr}. As discussed in Section~\ref{sec:exp_rl}, developing RL algorithms specifically tailored to interleaved verbal–visual generation may further improve world-model fidelity during reasoning and incentivize the emergence of stronger and intriguing world-modeling capabilities.

The analysis of emergent representations for implicit world modeling in Figure~\ref{fig:implict_wm} is intriguing but preliminary. We hope this result will rekindle interest in probing approaches \cite{li2023emergent} for interpreting the latent representations learned by different models. In particular, we are interested in comparing the internal representations of VLMs and UMMs, as the latter may capture complementary aspects of world knowledge through training for multimodal generation.

Artificial intelligence is increasingly being embodied in the physical world \cite{gupta2021embodied}. Our work, particularly the visual superiority hypothesis, suggests that learning visual world models is therefore essential for embodied intelligence. Visual world modeling enables embodied agents to better understand their environments, from imagining occluded regions to interpreting user intentions from an egocentric perspective, thereby supporting more reliable and natural everyday services. It also facilitates planning and decision-making by allowing agents to mentally simulate the precise outcomes of potential actions, leading to more effective interaction with the world. Rather than relying on loosely coupled modules \cite{feng2025reflective} or performing only single-step reasoning \cite{zhao2025cot}, we envision a future direction in which flexible multimodal world modeling and reasoning, empowered by interleave verbal-visual generation within a unified model, form core components of physical and embodied AI. 

\section*{Acknowledgements}

We would like to thank Yanwei Li, Rui Yang, Ziyu Zhu, and Feng Cheng for their assistance in constructing some preliminary training and test data. We also appreciate Xinchen Zhang, Jianhua Zhu, Yifan Du, Yuezhou Ma, Xingzhuo Guo, Ningya Feng, Shangchen Miao, and many colleagues for their valuable discussion.

\clearpage

\bibliographystyle{plainnat}
\bibliography{main}

\begin{thebibliography}{77}
\providecommand{\natexlab}[1]{#1}
\providecommand{\url}[1]{\texttt{#1}}
\expandafter\ifx\csname urlstyle\endcsname\relax
  \providecommand{\doi}[1]{doi: #1}\else
  \providecommand{\doi}{doi: \begingroup \urlstyle{rm}\Url}\fi

\bibitem[Agarwal et~al.(2025)Agarwal, Ali, Bala, Balaji, Barker, Cai, Chattopadhyay, Chen, Cui, Ding, et~al.]{agarwal2025cosmos}
Niket Agarwal, Arslan Ali, Maciej Bala, Yogesh Balaji, Erik Barker, Tiffany Cai, Prithvijit Chattopadhyay, Yongxin Chen, Yin Cui, Yifan Ding, et~al.
\newblock Cosmos world foundation model platform for physical ai.
\newblock \emph{arXiv preprint arXiv:2501.03575}, 2025.

\bibitem[Alonso et~al.(2024)Alonso, Jelley, Micheli, Kanervisto, Storkey, Pearce, and Fleuret]{alonso2024diffusion}
Eloi Alonso, Adam Jelley, Vincent Micheli, Anssi Kanervisto, Amos~J Storkey, Tim Pearce, and Fran{\c{c}}ois Fleuret.
\newblock Diffusion for world modeling: Visual details matter in atari.
\newblock In \emph{NeurIPS}, 2024.

\bibitem[Assran et~al.(2023)Assran, Duval, Misra, Bojanowski, Vincent, Rabbat, LeCun, and Ballas]{assran2023self}
Mahmoud Assran, Quentin Duval, Ishan Misra, Piotr Bojanowski, Pascal Vincent, Michael Rabbat, Yann LeCun, and Nicolas Ballas.
\newblock Self-supervised learning from images with a joint-embedding predictive architecture.
\newblock In \emph{CVPR}, 2023.

\bibitem[Assran et~al.(2025)Assran, Bardes, Fan, Garrido, Howes, Muckley, Rizvi, Roberts, Sinha, Zholus, et~al.]{assran2025v}
Mido Assran, Adrien Bardes, David Fan, Quentin Garrido, Russell Howes, Matthew Muckley, Ammar Rizvi, Claire Roberts, Koustuv Sinha, Artem Zholus, et~al.
\newblock V-jepa 2: Self-supervised video models enable understanding, prediction and planning.
\newblock \emph{arXiv preprint arXiv:2506.09985}, 2025.

\bibitem[Bai et~al.(2025{\natexlab{a}})Bai, Cai, Chen, Chen, Chen, Cheng, Deng, Ding, Gao, Ge, et~al.]{bai2025qwen3vltechnicalreport}
Shuai Bai, Yuxuan Cai, Ruizhe Chen, Keqin Chen, Xionghui Chen, Zesen Cheng, Lianghao Deng, Wei Ding, Chang Gao, Chunjiang Ge, et~al.
\newblock Qwen3-vl technical report.
\newblock \emph{arXiv preprint arXiv:2511.21631}, 2025{\natexlab{a}}.

\bibitem[Bai et~al.(2025{\natexlab{b}})Bai, Chen, Liu, Wang, Ge, Song, Dang, Wang, Wang, Tang, et~al.]{bai2025qwen2}
Shuai Bai, Keqin Chen, Xuejing Liu, Jialin Wang, Wenbin Ge, Sibo Song, Kai Dang, Peng Wang, Shijie Wang, Jun Tang, et~al.
\newblock Qwen2.5-vl technical report.
\newblock \emph{arXiv preprint arXiv:2502.13923}, 2025{\natexlab{b}}.

\bibitem[Bardes et~al.(2024)Bardes, Garrido, Ponce, Chen, Rabbat, LeCun, Assran, and Ballas]{bardes2024revisiting}
Adrien Bardes, Quentin Garrido, Jean Ponce, Xinlei Chen, Michael Rabbat, Yann LeCun, Mahmoud Assran, and Nicolas Ballas.
\newblock Revisiting feature prediction for learning visual representations from video.
\newblock \emph{arXiv preprint arXiv:2404.08471}, 2024.

\bibitem[Cai et~al.(2025)Cai, Wang, Sun, Wang, Gu, Yin, Lin, Yang, Wei, Shi, et~al.]{cai2025has}
Zhongang Cai, Yubo Wang, Qingping Sun, Ruisi Wang, Chenyang Gu, Wanqi Yin, Zhiqian Lin, Zhitao Yang, Chen Wei, Xuanke Shi, et~al.
\newblock Has gpt-5 achieved spatial intelligence? an empirical study.
\newblock \emph{arXiv preprint arXiv:2508.13142}, 2025.

\bibitem[Chen et~al.(2025)Chen, Moutakanni, Chung, Bang, Ji, Bolourchi, and Fung]{chen2025planning}
Delong Chen, Theo Moutakanni, Willy Chung, Yejin Bang, Ziwei Ji, Allen Bolourchi, and Pascale Fung.
\newblock Planning with reasoning using vision language world model.
\newblock \emph{arXiv preprint arXiv:2509.02722}, 2025.

\bibitem[Copet et~al.(2025)Copet, Carbonneaux, Cohen, Gehring, Kahn, Kossen, Kreuk, McMilin, Meyer, Wei, et~al.]{copet2025cwm}
Jade Copet, Quentin Carbonneaux, Gal Cohen, Jonas Gehring, Jacob Kahn, Jannik Kossen, Felix Kreuk, Emily McMilin, Michel Meyer, Yuxiang Wei, et~al.
\newblock Cwm: An open-weights llm for research on code generation with world models.
\newblock \emph{arXiv preprint arXiv:2510.02387}, 2025.

\bibitem[Craik(1967)]{craik1967nature}
Kenneth James~Williams Craik.
\newblock \emph{The nature of explanation}, volume 445.
\newblock CUP Archive, 1967.

\bibitem[DeepMind(2025)]{genie3}
Google DeepMind.
\newblock Genie 3: A new frontier for world models.
\newblock 2025.

\bibitem[Deng et~al.(2025)Deng, Zhu, Li, Gou, Li, Wang, Zhong, Yu, Nie, Song, et~al.]{deng2025emerging}
Chaorui Deng, Deyao Zhu, Kunchang Li, Chenhui Gou, Feng Li, Zeyu Wang, Shu Zhong, Weihao Yu, Xiaonan Nie, Ziang Song, et~al.
\newblock Emerging properties in unified multimodal pretraining.
\newblock \emph{arXiv preprint arXiv:2505.14683}, 2025.

\bibitem[Du et~al.(2025)Du, Zhou, Min, Ling, Zhao, and Wu]{du2025revisiting}
Yifan Du, Kun Zhou, Yingqian Min, Yue Ling, Wayne~Xin Zhao, and Youbin Wu.
\newblock Revisiting the necessity of lengthy chain-of-thought in vision-centric reasoning generalization.
\newblock \emph{arXiv preprint arXiv:2511.22586}, 2025.

\bibitem[Feng et~al.(2025)Feng, Han, Yang, Yue, Levine, and Luo]{feng2025reflective}
Yunhai Feng, Jiaming Han, Zhuoran Yang, Xiangyu Yue, Sergey Levine, and Jianlan Luo.
\newblock Reflective planning: Vision-language models for multi-stage long-horizon robotic manipulation.
\newblock \emph{arXiv preprint arXiv:2502.16707}, 2025.

\bibitem[Forrester(1971)]{forrester1971counterintuitive}
Jay~W Forrester.
\newblock Counterintuitive behavior of social systems.
\newblock \emph{Theory and decision}, 2\penalty0 (2):\penalty0 109--140, 1971.

\bibitem[Gu et~al.(2025)Gu, Hao, Wang, Li, Shieh, Choi, Krishna, and Cheng]{gu2025thinkmorph}
Jiawei Gu, Yunzhuo Hao, Huichen~Will Wang, Linjie Li, Michael~Qizhe Shieh, Yejin Choi, Ranjay Krishna, and Yu~Cheng.
\newblock Thinkmorph: Emergent properties in multimodal interleaved chain-of-thought reasoning.
\newblock \emph{arXiv preprint arXiv:2510.27492}, 2025.

\bibitem[Guo et~al.(2025{\natexlab{a}})Guo, Yang, Zhang, Song, Wang, Zhu, Xu, Zhang, Ma, Bi, et~al.]{guo2025deepseek}
Daya Guo, Dejian Yang, Haowei Zhang, Junxiao Song, Peiyi Wang, Qihao Zhu, Runxin Xu, Ruoyu Zhang, Shirong Ma, Xiao Bi, et~al.
\newblock Deepseek-r1 incentivizes reasoning in llms through reinforcement learning.
\newblock \emph{Nature}, 2025{\natexlab{a}}.

\bibitem[Guo et~al.(2025{\natexlab{b}})Guo, Wu, Zhu, Leng, Shi, Chen, Fan, Wang, Jiang, Wang, et~al.]{guo2025seed1}
Dong Guo, Faming Wu, Feida Zhu, Fuxing Leng, Guang Shi, Haobin Chen, Haoqi Fan, Jian Wang, Jianyu Jiang, Jiawei Wang, et~al.
\newblock Seed1.5-vl technical report.
\newblock \emph{arXiv preprint arXiv:2505.07062}, 2025{\natexlab{b}}.

\bibitem[Guo et~al.(2025{\natexlab{c}})Guo, Chu, Yang, Mo, Shen, Li, Lin, Zhang, Chen, Zhang, et~al.]{guo2025rbench}
Meng-Hao Guo, Xuanyu Chu, Qianrui Yang, Zhe-Han Mo, Yiqing Shen, Pei-lin Li, Xinjie Lin, Jinnian Zhang, Xin-Sheng Chen, Yi~Zhang, et~al.
\newblock Rbench-v: A primary assessment for visual reasoning models with multi-modal outputs.
\newblock \emph{arXiv preprint arXiv:2505.16770}, 2025{\natexlab{c}}.

\bibitem[Guo et~al.(2025{\natexlab{d}})Guo, Zhang, Li, Zhang, Chen, Wang, Feng, Pei, and Heng]{guo2025thinking}
Ziyu Guo, Renrui Zhang, Hongyu Li, Manyuan Zhang, Xinyan Chen, Sifan Wang, Yan Feng, Peng Pei, and Pheng-Ann Heng.
\newblock Thinking-while-generating: Interleaving textual reasoning throughout visual generation.
\newblock \emph{arXiv preprint arXiv:2511.16671}, 2025{\natexlab{d}}.

\bibitem[Guo et~al.(2025{\natexlab{e}})Guo, Zhang, Tong, Zhao, Huang, Zhang, Zhang, Liu, Zhang, Gao, et~al.]{guo2025can}
Ziyu Guo, Renrui Zhang, Chengzhuo Tong, Zhizheng Zhao, Rui Huang, Haoquan Zhang, Manyuan Zhang, Jiaming Liu, Shanghang Zhang, Peng Gao, et~al.
\newblock Can we generate images with cot? let's verify and reinforce image generation step by step.
\newblock In \emph{CVPR}, 2025{\natexlab{e}}.

\bibitem[Gupta et~al.(2021)Gupta, Savarese, Ganguli, and Fei-Fei]{gupta2021embodied}
Agrim Gupta, Silvio Savarese, Surya Ganguli, and Li~Fei-Fei.
\newblock Embodied intelligence via learning and evolution.
\newblock \emph{Nature Communications}, 2021.

\bibitem[Ha and Schmidhuber(2018)]{ha2018world}
David Ha and J{\"u}rgen Schmidhuber.
\newblock World models.
\newblock \emph{arXiv preprint arXiv:1803.10122}, 2\penalty0 (3), 2018.

\bibitem[Hafner et~al.(2025)Hafner, Pasukonis, Ba, and Lillicrap]{hafner2025mastering}
Danijar Hafner, Jurgis Pasukonis, Jimmy Ba, and Timothy Lillicrap.
\newblock Mastering diverse control tasks through world models.
\newblock \emph{Nature}, 2025.

\bibitem[Hansen et~al.(2022)Hansen, Wang, and Su]{hansen2022temporal}
Nicklas Hansen, Xiaolong Wang, and Hao Su.
\newblock Temporal difference learning for model predictive control.
\newblock In \emph{ICML}, 2022.

\bibitem[Huh et~al.(2024)Huh, Cheung, Wang, and Isola]{huh2024platonic}
Minyoung Huh, Brian Cheung, Tongzhou Wang, and Phillip Isola.
\newblock The platonic representation hypothesis.
\newblock In \emph{ICML}, 2024.

\bibitem[Hurst et~al.(2024)Hurst, Lerer, Goucher, Perelman, Ramesh, Clark, Ostrow, Welihinda, Hayes, Radford, et~al.]{hurst2024gpt}
Aaron Hurst, Adam Lerer, Adam~P Goucher, Adam Perelman, Aditya Ramesh, Aidan Clark, AJ~Ostrow, Akila Welihinda, Alan Hayes, Alec Radford, et~al.
\newblock Gpt-4o system card.
\newblock \emph{arXiv preprint arXiv:2410.21276}, 2024.

\bibitem[Ivanitskiy et~al.(2023)Ivanitskiy, Shah, Spies, Räuker, Valentine, Rager, Quirke, Mathwin, Corlouer, Behn, and Fung]{maze-dataset}
Michael~Igorevich Ivanitskiy, Rusheb Shah, Alex~F. Spies, Tilman Räuker, Dan Valentine, Can Rager, Lucia Quirke, Chris Mathwin, Guillaume Corlouer, Cecilia~Diniz Behn, and Samy~Wu Fung.
\newblock A configurable library for generating and manipulating maze datasets.
\newblock \emph{arXiv preprint arXiv:2309.10498}, 2023.

\bibitem[Johnson et~al.(2017)Johnson, Hariharan, Van Der~Maaten, Fei-Fei, Lawrence~Zitnick, and Girshick]{johnson2017clevr}
Justin Johnson, Bharath Hariharan, Laurens Van Der~Maaten, Li~Fei-Fei, C~Lawrence~Zitnick, and Ross Girshick.
\newblock Clevr: A diagnostic dataset for compositional language and elementary visual reasoning.
\newblock In \emph{CVPR}, 2017.

\bibitem[Johnson-Laird(1983)]{johnson1983mental}
PN~Johnson-Laird.
\newblock \emph{Mental models: Towards a cognitive science of language, inference, and consciousness}.
\newblock Harvard University Press, 1983.

\bibitem[Lakoff and N{\'u}{\~n}ez(2000)]{lakoff2000mathematics}
George Lakoff and Rafael N{\'u}{\~n}ez.
\newblock \emph{Where mathematics comes from}, volume~6.
\newblock New York: Basic Books, 2000.

\bibitem[Landy and Goldstone(2007)]{landy2007abstract}
David Landy and Robert~L Goldstone.
\newblock How abstract is symbolic thought?
\newblock \emph{Journal of Experimental Psychology: Learning, Memory, and Cognition}, 33\penalty0 (4):\penalty0 720, 2007.

\bibitem[LeCun(2022)]{lecun2022path}
Yann LeCun.
\newblock A path towards autonomous machine intelligence version 0.9. 2, 2022-06-27.
\newblock \emph{Open Review}, 62\penalty0 (1):\penalty0 1--62, 2022.

\bibitem[Li et~al.(2025{\natexlab{a}})Li, Wang, Fu, Yue, Cai, Zhu, Liu, Guo, Neiswanger, Huang, et~al.]{li2025zebra}
Ang Li, Charles Wang, Deqing Fu, Kaiyu Yue, Zikui Cai, Wang~Bill Zhu, Ollie Liu, Peng Guo, Willie Neiswanger, Furong Huang, et~al.
\newblock Zebra-cot: A dataset for interleaved vision language reasoning.
\newblock \emph{arXiv preprint arXiv:2507.16746}, 2025{\natexlab{a}}.

\bibitem[Li et~al.(2025{\natexlab{b}})Li, Wu, Zhang, Xia, Mao, Dong, Vuli{\'c}, and Wei]{li2025imagine}
Chengzu Li, Wenshan Wu, Huanyu Zhang, Yan Xia, Shaoguang Mao, Li~Dong, Ivan Vuli{\'c}, and Furu Wei.
\newblock Imagine while reasoning in space: Multimodal visualization-of-thought.
\newblock In \emph{ICML}, 2025{\natexlab{b}}.

\bibitem[Li et~al.(2023)Li, Hopkins, Bau, Vi{\'e}gas, Pfister, and Wattenberg]{li2023emergent}
Kenneth Li, Aspen~K Hopkins, David Bau, Fernanda Vi{\'e}gas, Hanspeter Pfister, and Martin Wattenberg.
\newblock Emergent world representations: Exploring a sequence model trained on a synthetic task.
\newblock In \emph{ICLR}, 2023.

\bibitem[Liang et~al.(2025)Liang, Chow, Li, Ma, Wang, Mao, Chen, Gu, Wang, and Huang]{liang2025rover}
Yongyuan Liang, Wei Chow, Feng Li, Ziqiao Ma, Xiyao Wang, Jiageng Mao, Jiuhai Chen, Jiatao Gu, Yue Wang, and Furong Huang.
\newblock Rover: Benchmarking reciprocal cross-modal reasoning for omnimodal generation.
\newblock \emph{arXiv preprint arXiv:2511.01163}, 2025.

\bibitem[Liu et~al.(2025)Liu, Liu, Liang, Li, Liu, Wang, Wan, Zhang, and Ouyang]{liu2025flow}
Jie Liu, Gongye Liu, Jiajun Liang, Yangguang Li, Jiaheng Liu, Xintao Wang, Pengfei Wan, Di~Zhang, and Wanli Ouyang.
\newblock Flow-grpo: Training flow matching models via online rl.
\newblock In \emph{NeurIPS}, 2025.

\bibitem[Loomis et~al.(1991)Loomis, Klatzky, Golledge, and CicineIli]{loomis1991mental}
JM~Loomis, RL~Klatzky, RG~Golledge, and JG~CicineIli.
\newblock Mental models, psychology of.
\newblock \emph{Psychology}, 14:\penalty0 56--89, 1991.

\bibitem[Ma et~al.(2025)Ma, Liu, Chen, Liu, Wu, Wu, Pan, Xie, Zhang, Yu, et~al.]{ma2025janusflow}
Yiyang Ma, Xingchao Liu, Xiaokang Chen, Wen Liu, Chengyue Wu, Zhiyu Wu, Zizheng Pan, Zhenda Xie, Haowei Zhang, Xingkai Yu, et~al.
\newblock Janusflow: Harmonizing autoregression and rectified flow for unified multimodal understanding and generation.
\newblock In \emph{CVPR}, 2025.

\bibitem[Mon-Williams et~al.(2025)Mon-Williams, Li, Long, Du, and Lucas]{mon2025embodied}
Ruaridh Mon-Williams, Gen Li, Ran Long, Wenqian Du, and Christopher~G Lucas.
\newblock Embodied large language models enable robots to complete complex tasks in unpredictable environments.
\newblock \emph{Nature Machine Intelligence}, 2025.

\bibitem[Monzel and Reuter(2024)]{monzel2024s}
Merlin Monzel and Martin Reuter.
\newblock Where’s wanda? the influence of visual imagery vividness on visual search speed measured by means of hidden object pictures.
\newblock \emph{Attention, Perception, \& Psychophysics}, 86\penalty0 (1):\penalty0 22--27, 2024.

\bibitem[Norman(2014)]{norman2014some}
Donald~A Norman.
\newblock Some observations on mental models.
\newblock In \emph{Mental models}, pages 7--14. Psychology Press, 2014.

\bibitem[Paivio(1990)]{paivio1990mental}
Allan Paivio.
\newblock \emph{Mental representations: A dual coding approach}.
\newblock Oxford university press, 1990.

\bibitem[Pan et~al.(2025)Pan, Shukla, Singh, Zhao, Mishra, Wang, Xu, Chen, Li, Juefei-Xu, et~al.]{pan2025transfer}
Xichen Pan, Satya~Narayan Shukla, Aashu Singh, Zhuokai Zhao, Shlok~Kumar Mishra, Jialiang Wang, Zhiyang Xu, Jiuhai Chen, Kunpeng Li, Felix Juefei-Xu, et~al.
\newblock Transfer between modalities with metaqueries.
\newblock \emph{arXiv preprint arXiv:2504.06256}, 2025.

\bibitem[Rombach et~al.(2022)Rombach, Blattmann, Lorenz, Esser, and Ommer]{rombach2022high}
Robin Rombach, Andreas Blattmann, Dominik Lorenz, Patrick Esser, and Bj{\"o}rn Ommer.
\newblock High-resolution image synthesis with latent diffusion models.
\newblock In \emph{CVPR}, 2022.

\bibitem[Schrittwieser et~al.(2020)Schrittwieser, Antonoglou, Hubert, Simonyan, Sifre, Schmitt, Guez, Lockhart, Hassabis, Graepel, et~al.]{schrittwieser2020mastering}
Julian Schrittwieser, Ioannis Antonoglou, Thomas Hubert, Karen Simonyan, Laurent Sifre, Simon Schmitt, Arthur Guez, Edward Lockhart, Demis Hassabis, Thore Graepel, et~al.
\newblock Mastering atari, go, chess and shogi by planning with a learned model.
\newblock \emph{Nature}, 2020.

\bibitem[Schulze~Buschoff et~al.(2025)Schulze~Buschoff, Akata, Bethge, and Schulz]{schulze2025visual}
Luca~M Schulze~Buschoff, Elif Akata, Matthias Bethge, and Eric Schulz.
\newblock Visual cognition in multimodal large language models.
\newblock \emph{Nature Machine Intelligence}, 2025.

\bibitem[Seedream et~al.(2025)Seedream, Chen, Gao, Gong, Guo, Guo, Guo, Hou, Huang, Huang, et~al.]{seedream2025seedream}
Team Seedream, Yunpeng Chen, Yu~Gao, Lixue Gong, Meng Guo, Qiushan Guo, Zhiyao Guo, Xiaoxia Hou, Weilin Huang, Yixuan Huang, et~al.
\newblock Seedream 4.0: Toward next-generation multimodal image generation.
\newblock \emph{arXiv preprint arXiv:2509.20427}, 2025.

\bibitem[Shi et~al.(2025{\natexlab{a}})Shi, Yu, Fang, Ren, Wang, Zhou, Tian, Fu, Hu, Lu, et~al.]{shi2025mathcanvas}
Weikang Shi, Aldrich Yu, Rongyao Fang, Houxing Ren, Ke~Wang, Aojun Zhou, Changyao Tian, Xinyu Fu, Yuxuan Hu, Zimu Lu, et~al.
\newblock Mathcanvas: Intrinsic visual chain-of-thought for multimodal mathematical reasoning.
\newblock \emph{arXiv preprint arXiv:2510.14958}, 2025{\natexlab{a}}.

\bibitem[Shi et~al.(2025{\natexlab{b}})Shi, Dong, Ding, Wang, Zhu, Zhou, Liu, Tian, Wang, Wang, et~al.]{shi2025realunify}
Yang Shi, Yuhao Dong, Yue Ding, Yuran Wang, Xuanyu Zhu, Sheng Zhou, Wenting Liu, Haochen Tian, Rundong Wang, Huanqian Wang, et~al.
\newblock Realunify: Do unified models truly benefit from unification? a comprehensive benchmark.
\newblock \emph{arXiv preprint arXiv:2509.24897}, 2025{\natexlab{b}}.

\bibitem[Swanson et~al.(2025)Swanson, Wu, Bulaong, Pak, and Zou]{swanson2025virtual}
Kyle Swanson, Wesley Wu, Nash~L Bulaong, John~E Pak, and James Zou.
\newblock The virtual lab of ai agents designs new sars-cov-2 nanobodies.
\newblock \emph{Nature}, 2025.

\bibitem[Team(2024)]{team2024chameleon}
Chameleon Team.
\newblock Chameleon: Mixed-modal early-fusion foundation models.
\newblock \emph{arXiv preprint arXiv:2405.09818}, 2024.

\bibitem[Tong et~al.(2025{\natexlab{a}})Tong, Tang, Li, Mou, Zhang, Zhao, Wen, Song, Zhan, Lu, et~al.]{tong2025gamerlsynthesizingmultimodalverifiable}
Jingqi Tong, Jixin Tang, Hangcheng Li, Yurong Mou, Ming Zhang, Jun Zhao, Yanbo Wen, Fan Song, Jiahao Zhan, Yuyang Lu, et~al.
\newblock Game-rl: Synthesizing multimodal verifiable game data to boost vlms' general reasoning.
\newblock \emph{arXiv preprint arXiv:2505.13886}, 2025{\natexlab{a}}.

\bibitem[Tong et~al.(2025{\natexlab{b}})Tong, Fan, Li, Xiong, Chen, Sinha, Rabbat, LeCun, Xie, and Liu]{tong2025metamorph}
Shengbang Tong, David Fan, Jiachen Li, Yunyang Xiong, Xinlei Chen, Koustuv Sinha, Michael Rabbat, Yann LeCun, Saining Xie, and Zhuang Liu.
\newblock Metamorph: Multimodal understanding and generation via instruction tuning.
\newblock In \emph{ICCV}, 2025{\natexlab{b}}.

\bibitem[Trinh et~al.(2024)Trinh, Wu, Le, He, and Luong]{trinh2024solving}
Trieu~H Trinh, Yuhuai Wu, Quoc~V Le, He~He, and Thang Luong.
\newblock Solving olympiad geometry without human demonstrations.
\newblock \emph{Nature}, 2024.

\bibitem[Tu et~al.(2025)Tu, Schaekermann, Palepu, Saab, Freyberg, Tanno, Wang, Li, Amin, Cheng, et~al.]{tu2025towards}
Tao Tu, Mike Schaekermann, Anil Palepu, Khaled Saab, Jan Freyberg, Ryutaro Tanno, Amy Wang, Brenna Li, Mohamed Amin, Yong Cheng, et~al.
\newblock Towards conversational diagnostic artificial intelligence.
\newblock \emph{Nature}, 2025.

\bibitem[Wang et~al.(2025{\natexlab{a}})Wang, Zhang, Wang, Gao, Li, Wang, Chen, Wan, Lu, Yang, et~al.]{wang2025vagen}
Kangrui Wang, Pingyue Zhang, Zihan Wang, Yaning Gao, Linjie Li, Qineng Wang, Hanyang Chen, Chi Wan, Yiping Lu, Zhengyuan Yang, et~al.
\newblock Vagen: Reinforcing world model reasoning for multi-turn vlm agents.
\newblock In \emph{NeurIPS}, 2025{\natexlab{a}}.

\bibitem[Wang et~al.(2024{\natexlab{a}})Wang, Todd, Xiao, Yuan, C{\^o}t{\'e}, Clark, and Jansen]{wang2024can}
Ruoyao Wang, Graham Todd, Ziang Xiao, Xingdi Yuan, Marc-Alexandre C{\^o}t{\'e}, Peter Clark, and Peter Jansen.
\newblock Can language models serve as text-based world simulators?
\newblock In \emph{ACL}, 2024{\natexlab{a}}.

\bibitem[Wang et~al.(2025{\natexlab{b}})Wang, Sun, Deng, Shao, Pei, Tian, Zhang, and Wang]{wang2025spatialviz}
Siting Wang, Luoyang Sun, Cheng Deng, Kun Shao, Minnan Pei, Zheng Tian, Haifeng Zhang, and Jun Wang.
\newblock Spatialviz-bench: Automatically generated spatial visualization reasoning tasks for mllms.
\newblock \emph{arXiv preprint arXiv:2507.07610}, 2025{\natexlab{b}}.

\bibitem[Wang et~al.(2024{\natexlab{b}})Wang, Zhang, Luo, Sun, Cui, Wang, Zhang, Wang, Li, Yu, et~al.]{wang2024emu3}
Xinlong Wang, Xiaosong Zhang, Zhengxiong Luo, Quan Sun, Yufeng Cui, Jinsheng Wang, Fan Zhang, Yueze Wang, Zhen Li, Qiying Yu, et~al.
\newblock Emu3: Next-token prediction is all you need.
\newblock \emph{arXiv preprint arXiv:2409.18869}, 2024{\natexlab{b}}.

\bibitem[Wu et~al.(2025{\natexlab{a}})Wu, Chen, Wu, Ma, Liu, Pan, Liu, Xie, Yu, Ruan, et~al.]{wu2025janus}
Chengyue Wu, Xiaokang Chen, Zhiyu Wu, Yiyang Ma, Xingchao Liu, Zizheng Pan, Wen Liu, Zhenda Xie, Xingkai Yu, Chong Ruan, et~al.
\newblock Janus: Decoupling visual encoding for unified multimodal understanding and generation.
\newblock In \emph{CVPR}, 2025{\natexlab{a}}.

\bibitem[Wu et~al.(2024)Wu, Yin, Feng, He, Li, Hao, and Long]{wu2024ivideogpt}
Jialong Wu, Shaofeng Yin, Ningya Feng, Xu~He, Dong Li, Jianye Hao, and Mingsheng Long.
\newblock ivideogpt: Interactive videogpts are scalable world models.
\newblock In \emph{NeurIPS}, 2024.

\bibitem[Wu et~al.(2025{\natexlab{b}})Wu, Yin, Feng, and Long]{wu2025rlvr}
Jialong Wu, Shaofeng Yin, Ningya Feng, and Mingsheng Long.
\newblock Rlvr-world: Training world models with reinforcement learning.
\newblock In \emph{NeurIPS}, 2025{\natexlab{b}}.

\bibitem[Xie et~al.(2025)Xie, Mao, Bai, Zhang, Wang, Lin, Gu, Chen, Yang, and Shou]{xie2024show}
Jinheng Xie, Weijia Mao, Zechen Bai, David~Junhao Zhang, Weihao Wang, Kevin~Qinghong Lin, Yuchao Gu, Zhijie Chen, Zhenheng Yang, and Mike~Zheng Shou.
\newblock Show-o: One single transformer to unify multimodal understanding and generation.
\newblock In \emph{ICLR}, 2025.

\bibitem[Xu et~al.(2025)Xu, Li, Zhou, Wan, Zhang, Korhonen, and Vuli{\'c}]{xu2025visual}
Yi~Xu, Chengzu Li, Han Zhou, Xingchen Wan, Caiqi Zhang, Anna Korhonen, and Ivan Vuli{\'c}.
\newblock Visual planning: Let's think only with images.
\newblock \emph{arXiv preprint arXiv:2505.11409}, 2025.

\bibitem[Yang et~al.(2025{\natexlab{a}})Yang, Zhu, Li, Huang, Yan, Zhou, Liu, Li, Li, Wang, et~al.]{yang2025visual}
Rui Yang, Ziyu Zhu, Yanwei Li, Jingjia Huang, Shen Yan, Siyuan Zhou, Zhe Liu, Xiangtai Li, Shuangye Li, Wenqian Wang, et~al.
\newblock Visual spatial tuning.
\newblock \emph{arXiv preprint arXiv:2511.05491}, 2025{\natexlab{a}}.

\bibitem[Yang et~al.(2025{\natexlab{b}})Yang, Xu, Xie, Yang, Li, Lin, Zhu, Chen, Duan, Yue, et~al.]{yang2025mmsi}
Sihan Yang, Runsen Xu, Yiman Xie, Sizhe Yang, Mo~Li, Jingli Lin, Chenming Zhu, Xiaochen Chen, Haodong Duan, Xiangyu Yue, et~al.
\newblock Mmsi-bench: A benchmark for multi-image spatial intelligence.
\newblock \emph{arXiv preprint arXiv:2505.23764}, 2025{\natexlab{b}}.

\bibitem[Yao et~al.(2025)Yao, Yu, Zhang, Wang, Cui, Zhu, Cai, Chen, Li, Zhao, et~al.]{yao2025efficient}
Yuan Yao, Tianyu Yu, Ao~Zhang, Chongyi Wang, Junbo Cui, Hongji Zhu, Tianchi Cai, Chi Chen, Haoyu Li, Weilin Zhao, et~al.
\newblock Efficient gpt-4v level multimodal large language model for deployment on edge devices.
\newblock \emph{Nature Communications}, 2025.

\bibitem[Yin et~al.(2025)Yin, Wang, Zhang, Zhang, Wang, Wang, Zhang, Chandrasegaran, Liu, Krishna, et~al.]{yin2025spatial}
Baiqiao Yin, Qineng Wang, Pingyue Zhang, Jianshu Zhang, Kangrui Wang, Zihan Wang, Jieyu Zhang, Keshigeyan Chandrasegaran, Han Liu, Ranjay Krishna, et~al.
\newblock Spatial mental modeling from limited views.
\newblock In \emph{Structural Priors for Vision Workshop at ICCV'25}, 2025.

\bibitem[Zhang et~al.(2025)Zhang, Chen, Liu, Xue, Liao, Liu, Wang, Ning, Chen, Fu, et~al.]{zhang2025agent}
Kai Zhang, Xiangchao Chen, Bo~Liu, Tianci Xue, Zeyi Liao, Zhihan Liu, Xiyao Wang, Yuting Ning, Zhaorun Chen, Xiaohan Fu, et~al.
\newblock Agent learning via early experience.
\newblock \emph{arXiv preprint arXiv:2510.08558}, 2025.

\bibitem[Zhao et~al.(2025{\natexlab{a}})Zhao, Lu, Kim, Fu, Zhang, Wu, Li, Ma, Han, Finn, et~al.]{zhao2025cot}
Qingqing Zhao, Yao Lu, Moo~Jin Kim, Zipeng Fu, Zhuoyang Zhang, Yecheng Wu, Zhaoshuo Li, Qianli Ma, Song Han, Chelsea Finn, et~al.
\newblock Cot-vla: Visual chain-of-thought reasoning for vision-language-action models.
\newblock In \emph{CVPR}, 2025{\natexlab{a}}.

\bibitem[Zhao et~al.(2025{\natexlab{b}})Zhao, Zhang, Tang, Zhu, Li, Chai, Zhang, Xia, Zhai, Yan, et~al.]{zhao2025envisioning}
Xiangyu Zhao, Peiyuan Zhang, Kexian Tang, Xiaorong Zhu, Hao Li, Wenhao Chai, Zicheng Zhang, Renqiu Xia, Guangtao Zhai, Junchi Yan, et~al.
\newblock Envisioning beyond the pixels: Benchmarking reasoning-informed visual editing.
\newblock In \emph{NeurIPS}, 2025{\natexlab{b}}.

\bibitem[Zhou et~al.(2025{\natexlab{a}})Zhou, Yu, Babu, Tirumala, Yasunaga, Shamis, Kahn, Ma, Zettlemoyer, and Levy]{zhou2024transfusion}
Chunting Zhou, Lili Yu, Arun Babu, Kushal Tirumala, Michihiro Yasunaga, Leonid Shamis, Jacob Kahn, Xuezhe Ma, Luke Zettlemoyer, and Omer Levy.
\newblock Transfusion: Predict the next token and diffuse images with one multi-modal model.
\newblock In \emph{ICLR}, 2025{\natexlab{a}}.

\bibitem[Zhou et~al.(2025{\natexlab{b}})Zhou, Tu, Wang, Wang, Muennighoff, Nie, Choi, Zou, Deng, Yan, et~al.]{zhou2025visualizing}
Yiyang Zhou, Haoqin Tu, Zijun Wang, Zeyu Wang, Niklas Muennighoff, Fan Nie, Yejin Choi, James Zou, Chaorui Deng, Shen Yan, et~al.
\newblock When visualizing is the first step to reasoning: Mira, a benchmark for visual chain-of-thought.
\newblock \emph{arXiv preprint arXiv:2511.02779}, 2025{\natexlab{b}}.

\bibitem[Zou et~al.(2025)Zou, Huang, Dong, Tian, Zheng, Liu, He, Liu, Qiao, and Liu]{zou2025uni}
Kai Zou, Ziqi Huang, Yuhao Dong, Shulin Tian, Dian Zheng, Hongbo Liu, Jingwen He, Bin Liu, Yu~Qiao, and Ziwei Liu.
\newblock Uni-mmmu: A massive multi-discipline multimodal unified benchmark.
\newblock \emph{arXiv preprint arXiv:2510.13759}, 2025.

\end{thebibliography}

\clearpage

\beginappendix

\section{Theorectical Analysis}
\label{app:theorectical}

\subsection{Informativeness}
\label{app:theorectical_info}

In this section, we provide the rigorous version of our world model-based chain-of-thought formulations, and proofs for Theorem~\ref{thm:kl_chain_rule} and Theorem~\ref{thm:mutual_info}.

\textbf{Formal problem setup and assumptions.} Given a question $Q$ and input images $I$, multimodal reasoning generates a chain-of-thought process $R$, followed by a final answer $A$. We explicitly formulate the reasoning process as an interleaving of logic reasoning steps and observations of the underlying MOMDP defined in Section~\ref{sec:momdp}: $R = (r_1, o_1), (r_2, o_2), \dots, (r_H, o_H)$ where $H$ denotes the (fixed) CoT length. For notation convenience, we denote the input image(s) as the initial observation $o_0$.

We assume that each MOMDP observation function admits a two-stage decomposition: $e_{\phi}(s) = g_{\phi_m}\left(f_{\phi_s}(s)\right), \Phi = \Phi_s \times \Phi_{m}$, where the inner modality-agnostic mapping $f_{\phi_s}$ (parameterized by $\phi_s \in \Phi_s$) extracts a \emph{slice} of the underlying state $s$, retaining only partial state information, and the outer modality-specific mapping $g_{\phi_m}$ (parameterized by $\phi_m \in \Phi_m$) renders the extracted slice into a particular observation modality. 

Under this decomposition, we assume that reasoning across different modalities of observations shares a common underlying \emph{oracle} reasoning process:
\begin{equation}
    \nonumber
    p(Q, \bar{s}_0, r_1, \bar{s}_1, \dots, r_H, \bar{s}_H, A) = p(Q) \left[\prod_{i=1}^H p(r_i | \bar s_{0:{i-1}}, r_{1:i-1}, Q) p(\bar s_i | \bar s_{0:i-1}, r_{1:i}, Q)\right] p(A | r_{1:H}, \bar s_{0:H}, Q), 
\end{equation}
where $\bar s_i=(s_i,{\phi_s}_i)\in \mathcal{S}\times\Phi_s$ denotes a modality-agnostic \emph{sliced} state. Each logic step $r_i$ is assumed to reason on sufficient sliced state information:
$
p(r_i \mid \bar s_{0:i-1}, r_{1:i-1}, Q)
=
p\!\left(r_i \mid f_{{\phi_s}_0}(s_0), \dots, f_{{\phi_s}_{i-1}}(s_{i-1}), r_{1:i-1}, Q\right),
$
and produces actionable outcomes that either (i) transit a previous world state $s_{j<i}$ via an implicit action $a_{i}$:
$
\bar s_{i} = (s_{i}, {\phi_s}_j), s_{i} \sim p(s_j, a_{i})
$
or (ii) query the same underlying world state with a new slice ${\phi_s}_{i}$, yielding
$
\bar s_{i} = (s_{j}, {\phi_s}_{i})
$
The oracle reasoning process is then rendered into a specific modality via
$
o_i = g_{\phi_m}\!\left(f_{{\phi_s}_i}(s_i)\right).
$
Unless otherwise specified, we abuse notation and use $s_i$ to denote $\bar s_i=(s_i,{\phi_s}_i)$ in the remainder of our analysis.

Given the above oracle CoT generation process, we learn a model $p_\theta$ whose joint distribution over CoTs and answers factorizes into a reasoning component and a world-modeling component:
\begin{align}
\label{app_eq:model_decomposition}
p_\theta(R, A | Q, I) = p_\theta(r_1, o_1, r_2, o_2, \dots, r_{H}, o_{H}, r_{H+1}| r_0, o_0) = \prod_{i=1}^{H+1} p_\theta(r_i| R_i) \prod_{i=1}^{H} p_\theta(o_i| \tilde R_i),
\end{align}
where we denote the question as $r_0$, the initial observation (input image(s)) as $o_0$, and the final answer as $r_{H+1}$. The CoT prefixes are defined as
$
R_i = (r_0,o_0,r_1,o_1,\dots,r_{i-1},o_{i-1}), \tilde R_i = (r_0,o_0,r_1,o_1,\dots,r_{i-1},o_{i-1},r_i).
$

\textbf{Proofs.} We provide proofs of Theorem~\ref{thm:kl_chain_rule} and Theorem~\ref{thm:mutual_info} below.

\begin{thm}[Restatement of Theorem~\ref{thm:kl_chain_rule}] For any observation modality $m$, the following inequality holds:
\begin{align}
\nonumber
\operatorname{KL}(p(A \mid Q, I)\mid\mid p_\theta(A \mid Q, I)) &\leq \operatorname{KL}(p(R, A \mid Q, I)\mid\mid p_\theta(R, A \mid Q, I))\\
= 
&\sum_{i=1}^{H+1} \underbrace{\mathbb{E}_p\left[\operatorname{KL}(p(r_i| R_i) \mid\mid  p_\theta(r_i| R_i) )\right]}_{\textnormal{reasoning errors}} + \sum_{i=1}^{H}  \underbrace{\mathbb{E}_p\left[\operatorname{KL}(p(o_i| \tilde R_i) \mid\mid p_\theta(o_i| \tilde R_i))\right]}_{\textnormal{world modeling errors}}.
\end{align}
\end{thm}

\begin{proof}
    The first inequality follows from the data processing inequality: marginalizing out $R$ cannot increase the KL divergence. For the equality, we apply the chain rule for KL divergence together with the CoT factorization in Eq.~\eqref{app_eq:model_decomposition}. In particular, substituting the factorizations of $p(R,A\mid Q,I)$ and $p_\theta(R,A\mid Q,I)$ into $\operatorname{KL}(p(R,A \mid Q,I)\,\|\,p_\theta(R,A \mid Q,I))$ leads to the stated decomposition.
\end{proof}

\begin{thm}[Restatement of Theorem~\ref{thm:mutual_info}] For any observation modality $m$, the reduction in reasoning uncertainty achieved by explicit world modeling satisfies:
\begin{enumerate}
    \item Reasoning uncertainty does not increase:
    $
    \mathbb{H}(r_i\mid o_{0}, r_{0:i-1}) - \mathbb{H}(r_i| R_i) = \mathbb{I}(o_{1:i-1}; r_i\mid o_0,r_{0:i-1}) \geq 0.
    $
    \item Uncertainty reduction is upper-bounded by both \textnormal{(i)} the information that observations provide about the underlying states and \textnormal{(ii)} the information that the reasoning step requires about those states:
    \begin{equation}
    \mathbb{I}(o_{1:i-1}; r_i\mid o_0,r_{0:i-1}) \leq \min\left(\mathbb{I}(o_{1:i-1}; s_{1:i-1}), \mathbb{I}(r_i; s_{0:i-1}, r_{0:i-1})\right).
    \label{app_eq:mi_ub}
    \end{equation}
\end{enumerate}
\end{thm}
\begin{proof}
The first property follows the definition and the non-negativity of mutual information. 

For the second property, denote the conditioning context as $C=(o_0, r_{0:i-1})$. Using the properties of ternary mutual information: 
$
\mathbb{I}(X;Y;Z) = \mathbb{I}(X;Y) - \mathbb{I}(X;Y\mid Z).
$
we obtain
\begin{align}
\nonumber
\mathbb{I}(o_{1:i-1}; r_i \mid C)
&=
\mathbb{I}(o_{1:i-1}; r_i \mid C) - \mathbb{I}(o_{1:i-1}; r_i \mid s_{1:i-1}, C) =
\mathbb{I}(s_{1:i-1}; o_{1:i-1}; r_i \mid C)\\ 
&= 
\mathbb{I}(o_{1:i-1}; s_{1:i-1} \mid C)
-
\mathbb{I}(o_{1:i-1}; s_{1:i-1} \mid r_i, C) \leq 
\mathbb{I}(o_{1:i-1}; s_{1:i-1}\mid C),
\label{app_eq:first_mutual_info_bound}
\end{align}
where $\mathbb{I}(o_{1:i-1}; r_i| s_{1:i-1}, C) = 0$ follows from the conditional independence $r_i \perp o_{1:i-1} \mid s_{1:i-1}$. 

Further, due to $o$ as the deterministic function of $s$, we have:
\begin{align}
\nonumber
\mathbb{I}(o_{1:i-1}; s_{1:i-1}\mid C) &= \mathbb{H}(o_{1:i-1} \mid C ) - \mathbb{H}(o_{1:i-1} \mid s_{1:i-1}, C ) \\ \nonumber &\leq \mathbb{H}(o_{1:i-1}) - \mathbb{H}(o_{1:i-1} \mid s_{1:i-1}) =  \mathbb{I}(o_{1:i-1}; s_{1:i-1}),
\end{align}
where $\mathbb{H}(o_{1:i-1} \mid s_{1:i-1}) = \mathbb{H}(o_{1:i-1} \mid s_{1:i-1}, C ) = 0.$

Symmetrically, we have:
\begin{align}
\nonumber
\mathbb{I}(o_{1:i-1}; r_i| C) &= \mathbb{I}(s_{1:i-1}; o_{1:i-1}; r_i \mid C) \leq \mathbb{I}(s_{1:i-1}; r_i| C) = \mathbb{H}(r_i | C ) - \mathbb{H}(r_i | s_{1:i-1}, C ) \\ \nonumber
&\leq \mathbb{H}(r_{i}) - \mathbb{H}(r_i | s_{0:i-1}, r_{0:i-1}) =  \mathbb{I}(r_i; s_{0:i-1}, r_{0:i-1}),
\end{align}
where $\mathbb{H}(r_i | s_{0:i-1}, r_{0:i-1}) \leq \mathbb{H}(r_i | s_{1:i-1}, o_0, r_{0:i-1})$ due to data processing inequality.

Combining the two upper bounds proves Eq.~\eqref{app_eq:mi_ub}.
\end{proof}

\begin{corollary}
\label{thm:deterministic}
If observations are fully informative about the underlying states, i.e., $\mathbb{H}(s_i \mid o_i)=0$ for all $i$, and the state transition dynamics are deterministic, then explicit world modeling provides no reduction in reasoning uncertainty:
$
\mathbb{I}(o_{1:i-1}; r_i \mid o_0, r_{0:i-1}) = 0.
$
\end{corollary}
\begin{proof}
By Eq.~\eqref{app_eq:first_mutual_info_bound}, we have
\begin{align}
\nonumber
    \mathbb{I}(o_{1:i-1}; r_i \mid o_0, r_{0:i-1}) \leq \mathbb{I}(o_{1:i-1}; s_{1:i-1}\mid o_0, r_{0:i-1}) \leq \mathbb{H}(s_{1:i-1} \mid o_0,r_{0:i-1}).
\end{align}
Under the assumption $\mathbb{H}(s_0 \mid o_0)=0$, the initial observation $o_0$ uniquely determines $s_0$. Moreover, under deterministic state transitions, the trajectory $s_{1:i-1}$ is uniquely determined by $(s_0, r_{1:i-1})$. Hence,
\begin{equation}
\nonumber
\mathbb{H}(s_{1:i-1} \mid o_0, r_{1:i-1})
=
\mathbb{H}(s_{1:i-1} \mid s_0, r_{1:i-1})
=0.    
\end{equation}
Therefore,
$
\mathbb{I}(o_{1:i-1}; r_i \mid o_0, r_{1:i-1}) = 0,
$
which proves the corollary.
\end{proof}

\textbf{Remarks.} Corollary~\ref{thm:deterministic} shows that in deterministic and fully observable environments, given sufficient data and model capacity, explicit world modeling provides no additional benefit. This theoretical result is consistent with our empirical findings on the simple maze task.

\subsection{Prior Knowledge}
\label{app:theoretical_transfer}

In this section, we first derive a generalization bound for transfer learning under distribution shift, and relate it to our perspective on prior knowledge in multimodal reasoning.

\subsubsection{General Transfer Learning Analysis}

\textbf{Problem setup.} A standard transfer learning setup involves a pre-training data distribution $P$ and the fine-tuning data distribution $Q$ over samples $(x,y) \in \mathcal{X} \times \mathcal{Y}$, and a loss function $\ell_\theta(x,y) \in [0, 1].$ Define the population risks $
\mathcal{L}_D(\theta) := \mathbb{E}_{(x,y)\sim D}[\ell_\theta(x,y)], D\in\{P,Q\}
$, and the population minimizers
$
\theta_D^\star \in \arg\min_{\theta\in\Theta}\mathcal{L}_D(\theta), D\in\{P,Q\}
$
We assume we can obtain $\theta_P^*$ as the pre-trained model given sufficient data and optimization. For a radius $r>0$, we then define the fine-tuning constraint set (local neighborhood around the pre-trained model)
\[
\Theta_r := \{\theta\in\Theta:\ \|\theta-\theta_P^*\|\le r\}.
\]
Given $n$ i.i.d.\ samples from $Q$: $S = \left\{(x_i, y_i)_{i=1}^n\right\}, (x_i, y_i)\sim Q$, the fine-tuned model $\theta_Q$ minimize empirical risk over $\Theta_r$,
$
 \widehat{\mathcal{L}}_Q(\theta) := \frac{1}{n}\sum_{i=1}^n[\ell_\theta(x_i,y_i)].
$
Our analysis focus on the \emph{excess risk} on $Q$:
$
\mathcal{E}_Q(\theta_Q) := \mathcal{L}_Q(\theta_Q)-\mathcal{L}_Q(\theta_Q^\star).
$

\textbf{From distribution shift to parameter drift.} We first derive how the distribution shift relates to the shift of the population minimizer.

\begin{lemma}[Uniform Loss Shift under Total Variation]
\label{lem:uniform_tv_shift}
For any subset $\mathcal{S}\subseteq\Theta$,
\[
\sup_{\theta\in\mathcal{S}}
\big|
\mathcal{L}_Q(\theta)-\mathcal{L}_P(\theta)
\big|
\le
\mathrm{TV}(P,Q).
\]
\end{lemma}

\begin{proof}
Fix any $\theta\in\mathcal{S}$ and define $f_\theta(h,a,o'):=\ell_\theta(h,a,o')\in[0,1]$.
By the definition of total variation and the standard inequality for bounded functions,
$
\big|\mathbb{E}_Q[f_\theta]-\mathbb{E}_P[f_\theta]\big|
\le
\mathrm{TV}(P,Q).
$
Taking the supremum over $\theta\in\mathcal{S}$ yields the claim.
\end{proof}

\begin{lemma}[Risk Proximity of $\theta_Q^\star$ under $P$]
\label{lem:risk_proximity}
\begin{equation}
\label{eq:risk_proximity}
\mathcal{L}_P(\theta_Q^\star)
\le
\mathcal{L}_P(\theta_P^\star) + 2\mathrm{TV}(P,Q).
\end{equation}
\end{lemma}

\begin{proof}
By Lemma~\ref{lem:uniform_tv_shift},
$
\mathcal{L}_P(\theta_Q^\star)\le \mathcal{L}_Q(\theta_Q^\star)+\mathrm{TV}(P,Q).
$
By optimality of $\theta_Q^\star$ on $Q$,
$
\mathcal{L}_Q(\theta_Q^\star)\le \mathcal{L}_Q(\theta_P^\star).
$
Applying Lemma~\ref{lem:uniform_tv_shift} again,
$
\mathcal{L}_Q(\theta_P^\star)\le \mathcal{L}_P(\theta_P^\star)+\mathrm{TV}(P,Q).
$
Chaining the three inequalities proves~\eqref{eq:risk_proximity}.
\end{proof}

\begin{assumption}[Local Quadratic Growth / Sharpness of $\mathcal{L}_P$]
\label{ass:local_sharpness}
There exists $\mu>0$ such that for all $\theta$ in a neighborhood containing $\theta_Q^\star$,
\[
\mathcal{L}_P(\theta)
\ge
\mathcal{L}_P(\theta_P^\star) + \frac{\mu}{2}\|\theta-\theta_P^\star\|^2.
\]
\end{assumption}

\begin{lemma}[Parameter Drift Controlled by $\mathrm{TV}(P,Q)$]
\label{lem:param_drift}
Under Assumption~\ref{ass:local_sharpness},
\begin{equation}
\label{eq:param_drift}
\|\theta_Q^\star-\theta_P^\star\|
\le
\sqrt{\frac{4}{\mu}\,\mathrm{TV}(P,Q)}.
\end{equation}
\end{lemma}

\begin{proof}
By Assumption~\ref{ass:local_sharpness} with $\theta=\theta_Q^\star$,
$
\mathcal{L}_P(\theta_Q^\star)
\ge
\mathcal{L}_P(\theta_P^\star) + \frac{\mu}{2}\|\theta_Q^\star-\theta_P^\star\|^2.
$
Rearranging,
$
\frac{\mu}{2}\|\theta_Q^\star-\theta_P^\star\|^2
\le
\mathcal{L}_P(\theta_Q^\star)-\mathcal{L}_P(\theta_P^\star).
$
Applying Lemma~\ref{lem:risk_proximity} yields
$
\frac{\mu}{2}\|\theta_Q^\star-\theta_P^\star\|^2
\le
2\mathrm{TV}(P,Q),
$
and hence~\eqref{eq:param_drift}.
\end{proof}

\textbf{Control of the bias term.} Recall the fine-tuning bias induced by restricting to $\Theta_r$:
$\varepsilon_{\mathrm{bias}}(r)
:=
\inf_{\theta\in\Theta_r}\mathcal{L}_Q(\theta) - \mathcal{L}_Q(\theta_Q^\star).
$

\begin{assumption}[$\mathcal{L}_Q$ is Locally Lipschitz]
\label{ass:lipschitz_LQ}
There exists $L_Q>0$ such that for all $\theta,\theta'\in\Theta_r$,
\[
|\mathcal{L}_Q(\theta)-\mathcal{L}_Q(\theta')|
\le
L_Q\|\theta-\theta'\|.
\]
\end{assumption}

\begin{thm}[Bias Bound via Distribution Shift]
\label{lem:bias_tv}
Under Assumption~\ref{ass:local_sharpness} and Assumption~\ref{ass:lipschitz_LQ},
\begin{equation}
\label{eq:bias_tv}
\varepsilon_{\mathrm{bias}}(r)
\le
L_Q\left(\sqrt{\frac{4}{\mu}\,\mathrm{TV}(P,Q)}-r\right)_+,
\end{equation}
where $(x)_+:=\max\{x,0\}$.
In particular, if
$
r \ge \sqrt{\frac{4}{\mu}\,\mathrm{TV}(P,Q)},
$
then $\varepsilon_{\mathrm{bias}}(r)=0$.
\end{thm}

\begin{proof}
If $r\ge \|\theta_Q^\star-\theta_P^\star\|$, then $\theta_Q^\star\in\Theta_r$ and thus
$\inf_{\theta\in\Theta_r}\mathcal{L}_Q(\theta)\le \mathcal{L}_Q(\theta_Q^\star)$, implying
$\varepsilon_{\mathrm{bias}}(r)=0$.

Now consider $r < \|\theta_Q^\star-\theta_P^\star\|$.
Let $\theta_r$ be the projection of $\theta_Q^\star$ onto the closed ball $\Theta_r$,
i.e., $\theta_r := \theta_P + r \cdot \frac{\theta_Q^\star-\theta_P^\star}{\|\theta_Q^\star-\theta_P^\star\|}$.
Then $\theta_r\in\Theta_r$ and
$
\|\theta_r-\theta_Q^\star\| = \|\theta_Q^\star-\theta_P^\star\|-r.
$
Therefore,
\[
\varepsilon_{\mathrm{bias}}(r)
=
\inf_{\theta\in\Theta_r}\mathcal{L}_Q(\theta) - \mathcal{L}_Q(\theta_Q^\star)
\le
\mathcal{L}_Q(\theta_r)-\mathcal{L}_Q(\theta_Q^\star)
\le
L_Q\|\theta_r-\theta_Q^\star\|
=
L_Q(\|\theta_Q^\star-\theta_P^\star\|-r).
\]
Using Lemma~\ref{lem:param_drift} to bound $\|\theta_Q^\star-\theta_P\|$ completes the proof of~\eqref{eq:bias_tv}.
\end{proof}

\textbf{Fine-tuning excess risk bound.} We then arrive at the final result:

\begin{thm}[Fine-tuning Excess Risk Bound with Shift-Controlled Bias]
\label{thm:finetune_nontrivial}
Assume Assumptions~\ref{ass:local_sharpness} and~\ref{ass:lipschitz_LQ} and uniform convergence over $\Theta_r$ holds: with probability at least $1-\delta$ over samples $S$,
  \[
  \sup_{\theta\in\Theta_r}\big|\mathcal{L}_Q(\theta)-\widehat{\mathcal{L}}_Q(\theta)\big|
  \le
  \varepsilon_{\mathrm{gen}}
  =
  O\!\left(\sqrt{\frac{\operatorname{Rad}_{Q,n}(\Theta_r)+\log(1/\delta)}{n}}\right),
  \]
  where $\operatorname{Rad}_{Q,n}(\Theta_r)$ is the Rademacher complexity of of the function class $\{\ell_\theta, \theta \in \Theta_r\}$ with respect to $Q$ for sample size $n$.
Then with probability at least $1-\delta$,
\begin{equation}
\label{eq:finetune_nontrivial}
\mathcal{E}_Q(\theta_Q)
\le
2\varepsilon_{\mathrm{gen}}
+
L_Q\left(\sqrt{\frac{4}{\mu}\,\mathrm{TV}(P,Q)}-r\right)_+.
\end{equation}
\end{thm}

\begin{proof}
Decompose the excess risk as
$
\mathcal{E}_Q(\theta_Q)
=
\Big(\mathcal{L}_Q(\theta_Q)-\inf_{\theta\in\Theta_r}\mathcal{L}_Q(\theta)\Big)
+
\varepsilon_{\mathrm{bias}}(r).
$
The first term is bounded by a standard ERM argument using uniform convergence:
$
\mathcal{L}_Q(\theta_Q)-\inf_{\theta\in\Theta_r}\mathcal{L}_Q(\theta)
\le
2\varepsilon_{\mathrm{gen}}.
$
The second term is bounded by Lemma~\ref{lem:bias_tv}. Combining the two bounds yields~\eqref{eq:finetune_nontrivial}.
\end{proof}

\subsubsection{Remarks on Multimodal Reasoning}

Theorem~\ref{thm:finetune_nontrivial} reveals a trade-off between modality complexity and distribution shift. This general transfer learning analysis can be instantiated in our setting of learning world models and reasoning policies. Specifically, training pairs $(x, y)$ can be instantiated as $((o_{0:i}, r_{0:{i+1}}), o_{i+1})$ for world modeling and $((o_{0:i}, r_{0:{i}}), r_{i})$ for reasoning, respectively. Crucially, the distribution shift between large-scale pre-training data and downstream tasks may differ substantially across modalities. For example, there are abundant visual demonstrations of paper folding on the Internet, whereas detailed verbal descriptions of folding dynamics are comparatively scarce. This suggests that downstream tasks should be formulated under the most appropriate observation modality for world modeling and reasoning—i.e., the modality that best aligns with pre-training data—in order to achieve stronger generalization at inference time and higher sample efficiency during post-training.

\section{Experiment Details}
\label{app:details}

\subsection{VisWorld-Eval and Training Data}
\label{app:task_details}

In this section, we elaborate on the construction of training and test data for each task in VisWorld-Eval.

\textbf{Paper folding.} This task involves folding a paper grid with varying grid sizes (3–8) and folding steps (1–4). After folding, holes of different shapes—circles, triangles, stars, diamonds, and squares—are punched into the paper. The model is then asked to predict the distribution of holes after the paper is completely unfolded, including queries such as the total number of holes, the number of holes of a specific shape, or the difference in counts between shapes. All test prompts are constructed at the highest difficulty level (grid size 8 with 4 folding steps). For SFT, we generate chain-of-thoughts using rule-based templates that follow a fixed procedure: unfold the paper step-by-step and then count the resulting holes by shape. These CoTs are then rewritten with Gemini 2.5 Pro to improve clarity and logical coherence. Under visual world modeling, we interleave reasoning steps with images of partially unfolded paper states. Under verbal world modeling, we represent intermediate states using two matrices encoding grid coverage status and hole shape at each position. Under implicit world modeling, we directly skip the explicit tracking of states from original CoTs.

\textbf{Multi-hop manipulation.} This task begins with an initial arrangement of several geometric objects (cubes, spheres, and cylinders) in various colors, rendered by Blender\footnote{\url{https://www.blender.org/}}. A sequence of text-based instructions is then provided, describing operations such as changing or swapping objects' color or shape, adding new objects, or removing existing ones. To ensure these commands can be interpreted unambiguously in a 3D space, the instructions consistently use relative spatial references, with each object uniquely identified by its combined color and shape attributes—for example: "Place a purple cylinder between the black sphere and the yellow cube." The model is asked to infer the resulting spatial layout. Queries may include the total number of objects of a specific shape, the directional relationship between two objects, or which object lies in a given direction relative to a reference object. Test prompts are constructed by varying both the number of initial objects (between 3 and 6) and the frequency (between 1 and 5) of different operation types. For SFT, chain-of-thought reasoning is generated using rule-based templates that simulate the stepwise execution of instructions before answering the final query, and these CoTs are subsequently refined with Gemini 2.5 Pro.

\textbf{Ball tracking.} This task features a red point-mass ball that moves at constant speed, reflects elastically off solid walls, and travels in the initial direction indicated by a green arrow. The model is asked to predict which numbered hole at the top of the image the ball will enter first. We generate input images with randomized resolution, initial ball position and direction, and a random number of holes (4–8). For test prompts, we select cases in which the ball trajectory reflects off at least one wall before entering a hole. For SFT, CoTs are generated by Seed 1.6, which is asked to explain the ball dynamics between adjacent frames.

\textbf{Sokoban.} Sokoban is a classic grid-based puzzle game. We generate instances with grid sizes ranging from 6 to 10, containing a single box and a target position. Test prompts are sampled from the same distribution as the training data. To construct CoTs, we use a search algorithm to compute an optimal solution path. To avoid excessively long trajectories, we render only key intermediate steps, including: (i) the player moving toward the box, (ii) pushing the box in a direction, and (iii) changing the pushing direction. To encourage reflective behavior, we additionally augment trajectories with randomized detours that involve walking into walls, reflecting, and backtracking to rejoin the optimal path. CoTs are generated by Seed 1.6, which explains the dynamics between adjacent frames. For visual world modeling, the rendered intermediate steps are interleaved with verbal CoTs. For pure verbal world modeling, these intermediate renderings are removed. For implicit world modeling, we additionally mask all explicit coordinates during CoTs with special tokens \texttt{[masked]}.

\textbf{Maze.} Maze is a classic grid-based puzzle task. We generate both training and test samples with a fixed grid size of $5\times 5$. To construct CoTs, we use rule-based templates followed by rewriting for improved naturalness. Under visual world modeling, rendered intermediate steps through points and lines are interleaved with verbal CoTs. The settings for verbal and implicit world modeling follow the same protocol as in Sokoban, with masking special tokens as \texttt{<point>[masked]</point>}.

\textbf{Cube 3-view projection.} This task considers stacks of colored cubes arranged on grids of varying sizes (3–5), with two cube colors. The input consists of one isometric view (either front-left or front-right) and two orthographic views of the stack. The question asks for the number of cubes of a specified color visible from another orthogonal view. Both the questions and answer choices account for ambiguity caused by occlusions, leading to uncertainty in the cube count. All test prompts are constructed using uniformly random grid sizes between 3 and 5. We generate CoTs using rule-based templates: the model first constructs the queried view, marks potentially occluded cubes using a third (auxiliary) color, and then counts cubes by color. These CoTs are subsequently rewritten by Gemini 2.5 Pro for improved naturalness. Under visual world modeling, we interleave reasoning steps with an image of the queried view. Under verbal world modeling, we represent intermediate views using character matrices, where different colors are encoded by different symbols.

\textbf{Real-world spatial reasoning.} For this real-world task, we directly adopt test samples from MMSI-Bench, focusing on camera–object and camera–region positional relationship questions. We construct training prompts following a pipeline similar to \citet{yang2025visual}. To obtain training CoTs, we run a visual-CoT model, which uses an SFT-trained BAGEL model for novel view synthesis as a tool. The resulting visual CoTs are subsequently filtered and rewritten by Gemini 2.5 Pro.

We summarize the training and test sample counts for each task in VisWorld-Eval, along with the corresponding original or referenced benchmarks, in Table~\ref{tab:benchmark}.

\begin{table}[htbp]
    \centering
    \caption{Overview of VisWorld-Eval and corresponding training data: features, statistics, and references.}
    \label{tab:benchmark}
    \small
    \begin{tabular}{lccccc}
        \toprule
         Task & Capability & Domain & \begin{tabular}[c]{@{}c@{}}Training\\ Samples\end{tabular} & \begin{tabular}[c]{@{}c@{}}Test\\ Samples\end{tabular} & Source/Reference \\ \midrule
         Paper folding  & Simulation& Synthetic & 2,357 & 480 & SpatialViz \cite{wang2025spatialviz} \\
         Multi-hop manipulation & Simulation & Synthetic & 2,000 & 480 & ZebraCoT \cite{li2025zebra}, CLEVR \cite{johnson2017clevr} \\
         Ball tracking  & Simulation& Synthetic & 2,254 & 1,024 & RBench-V \cite{guo2025rbench} \\
         Maze  & Simulation& Synthetic & 8,448 & 480 &  maze-dataset \cite{maze-dataset} \\ 
         Sokoban  & Simulation& Synthetic & 7,715 & 480 & GameRL \cite{tong2025gamerlsynthesizingmultimodalverifiable} \\
         Cube 3-view projection & Reconstruction & Synthetic & 2,500 & 480 & SpatialViz \cite{wang2025spatialviz} \\
         Real-world spatial reasoning  & Reconstruction & Real-world & 10,661 & 522 & MMSI-Bench \cite{yang2025mmsi} \\ \bottomrule
         \end{tabular}
\end{table}

Examples of training CoTs are presented in Figure~\ref{fig:paperfolding_data_showcase}, \ref{fig:ballgame_multihop_data_showcase}, \ref{fig:maze_sokoban_data_showcase}, \ref{fig:cube_data_showcase}, and \ref{fig:spatial_data_showcase}.

\subsection{Model Training}
\label{app:training}

We perform supervised fine-tuning (SFT) of BAGEL based on its official repository\footnote{\url{https://github.com/ByteDance-Seed/Bagel}}, using 8 GPUs,
 and conduct reinforcement learning from verifiable rewards (RLVR) using \texttt{verl}\footnote{\url{https://github.com/volcengine/verl}} on 64 GPUs. Hyperparameters for SFT and RLVR are reported in Table~\ref{tab:sft_hyperparameters} and Table~\ref{tab:rl_hyperparameters}, respectively.

\begin{table}[htbp]
    \centering
    \caption{Hyperparameters for supervised fine-tuning UMMs.}
    \label{tab:sft_hyperparameters}
    \begin{tabular}{lc}
         \toprule
         Hyperparameter & Value \\ \midrule
         Learning rate & $3\times 10^{-5}$ \\
         LR Schedule & Constant \\
         Optimizer & AdamW \\
         Loss weight (CE:MSE) & 1:10 \\
         Warm-up steps & 200 \\
         Training steps & 4000 \\
         Gen. resolution & \multicolumn{1}{l}{(256, 1024) for paper folding, cube 3-view} \\
                        & \multicolumn{1}{l}{(240, 1024) for multi-hop manipulation} \\
                        & \multicolumn{1}{l}{(256, 512) otherwise}  \\
         Und. resolution & (224, 980) \\
         Sequence length per rank & 32K \\
         Num. ranks & 8 \\
         \bottomrule
    \end{tabular}
    \vspace{10pt}
    \caption{Hyperparameters for reinforcement learning UMMs.}
    \label{tab:rl_hyperparameters}
    \begin{tabular}{lc}
        \toprule
         Hyperparameter & Value \\ \midrule
         Learning rate & $1\times 10^{-5}$ \\
         Batch size & 128 \\
         GRPO mini batch size & 32 \\
         Group size & 16 \\
         KL loss coefficient for visual gen. & 0.1 \\
         KL loss coefficient for verbal gen. & 0.0 \\
         \bottomrule
    \end{tabular}
\end{table}

\clearpage

\begin{figure}[p]
    \centering
    \includegraphics[width=\linewidth]{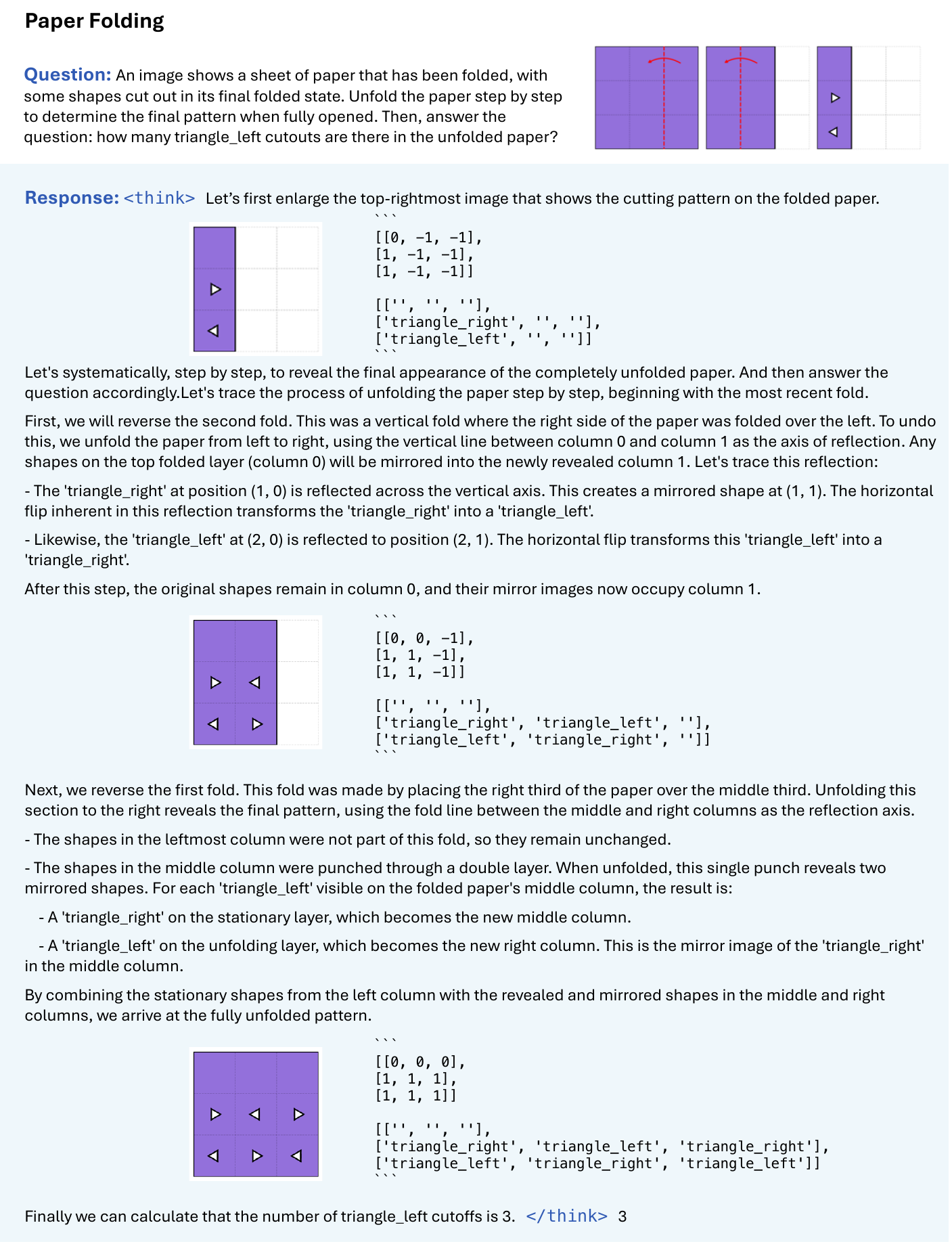}
    \caption{Examples of chain-of-thought SFT data for the paper folding task, under visual world modeling (left) and verbal world modeling (right).}
    \label{fig:paperfolding_data_showcase}
\end{figure}

\begin{figure}[p]
    \centering
    \includegraphics[width=\linewidth]{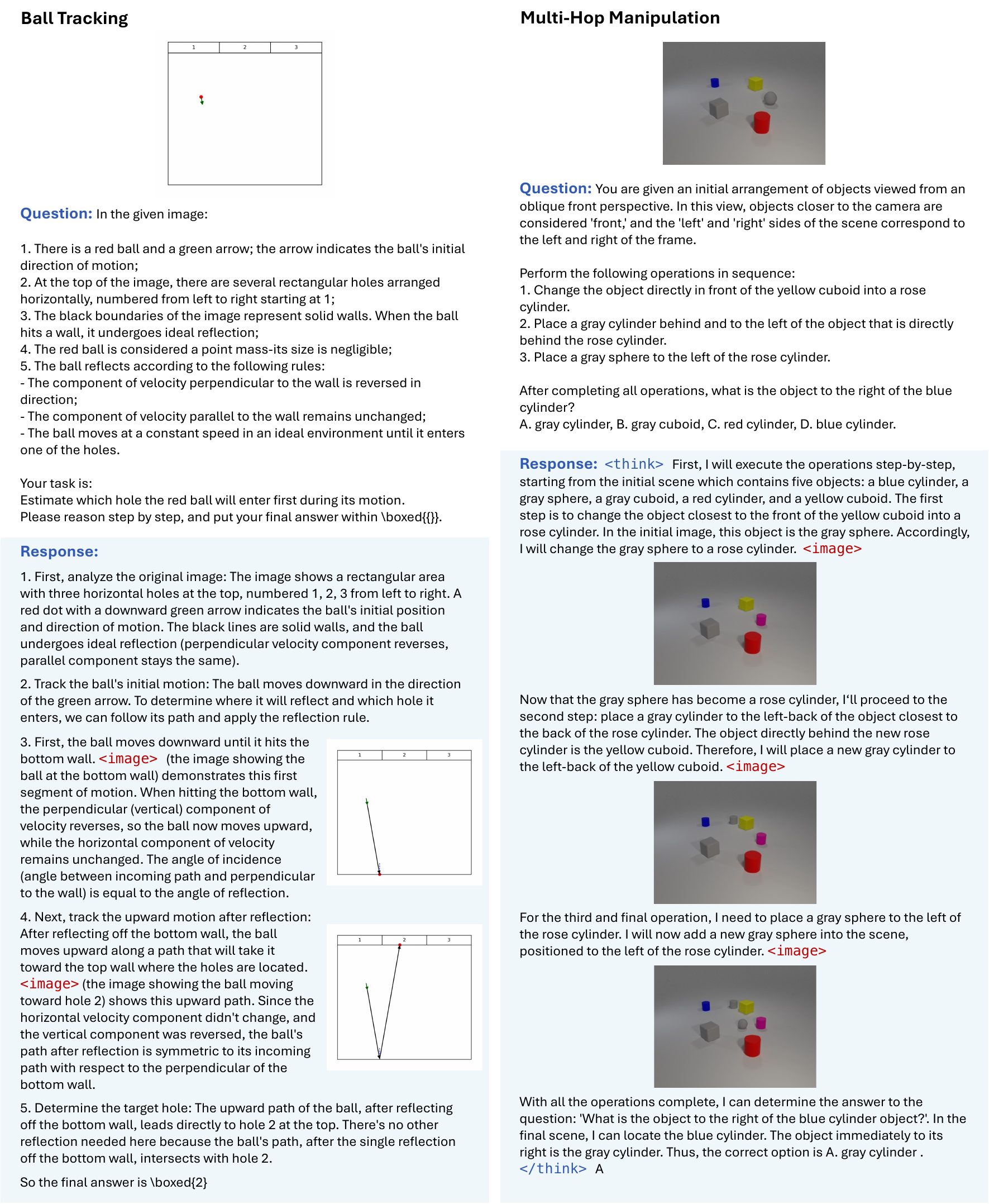}
    \caption{Examples of chain-of-thought SFT data for the ball tracking and multi-hop manipulation task.}
    \label{fig:ballgame_multihop_data_showcase}
\end{figure}

\begin{figure}[p]
    \centering
    \includegraphics[width=\linewidth]{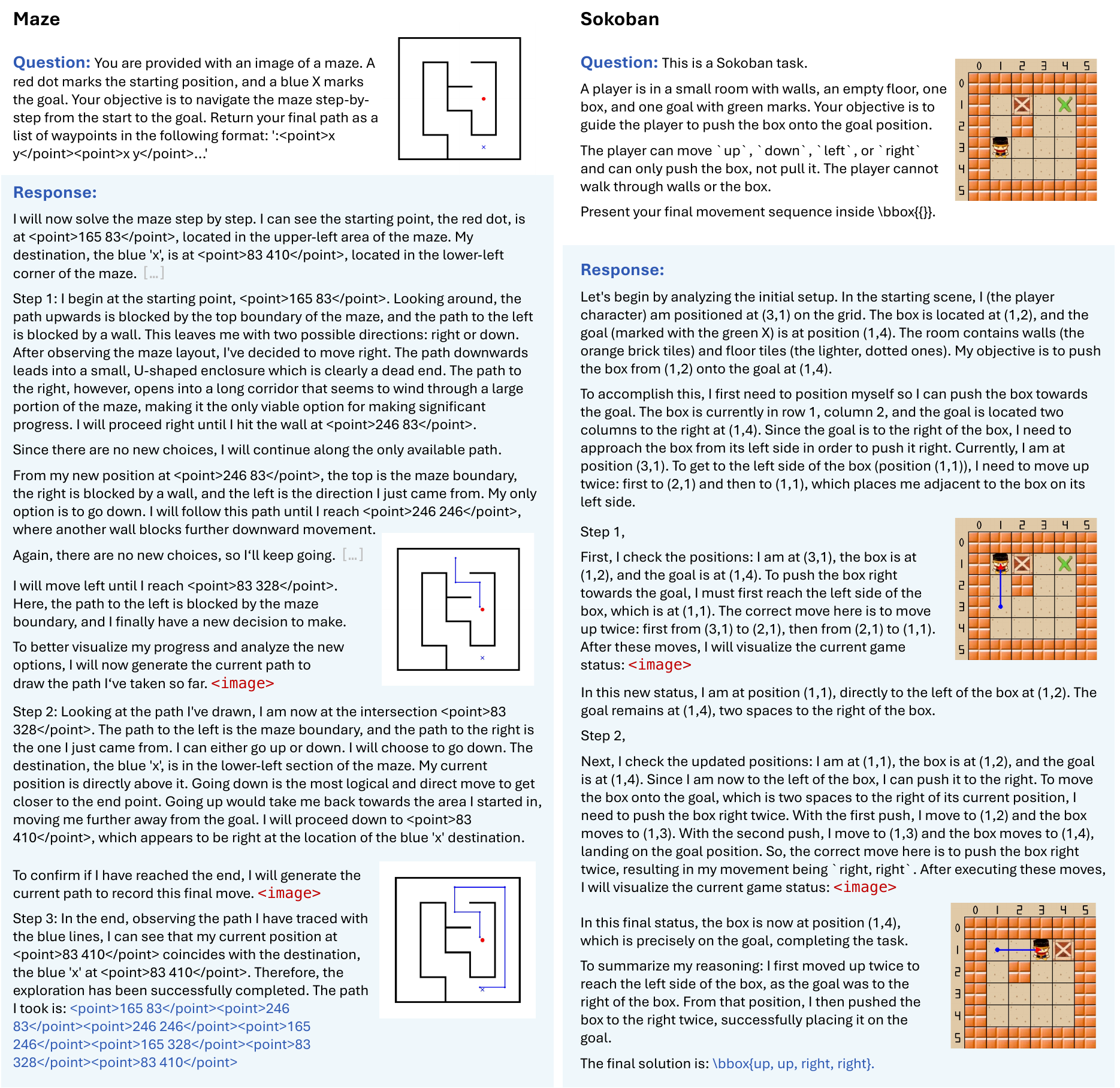}
    \caption{Examples of chain-of-thought SFT data for the maze and sokoban task.}
    \label{fig:maze_sokoban_data_showcase}
\end{figure}

\begin{figure}[p]
    \centering
    \includegraphics[width=\linewidth]{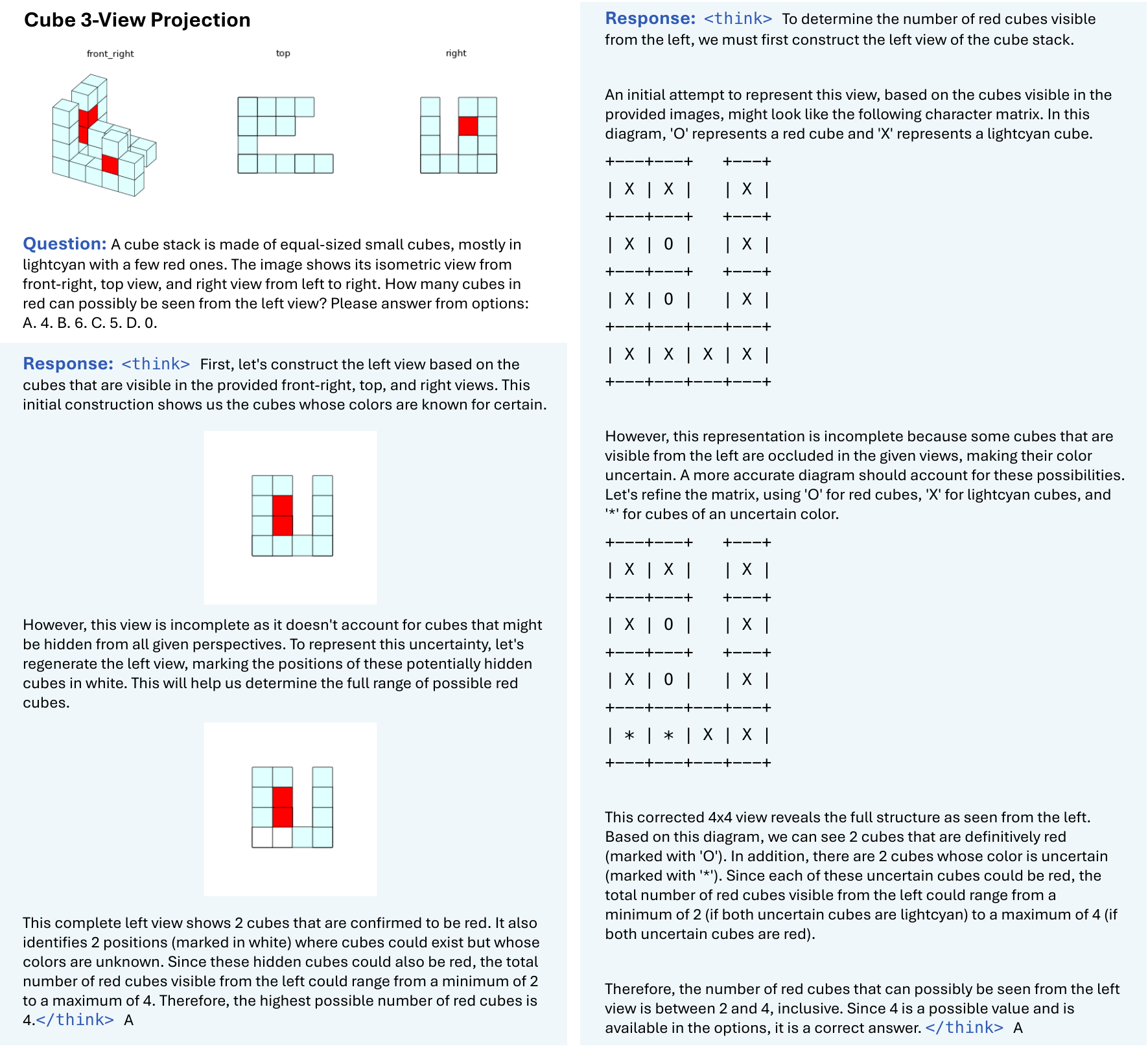}
    \caption{Examples of chain-of-thought SFT data for the cube 3-view projection task, under visual world modeling (left) and verbal world modeling (right).}
    \label{fig:cube_data_showcase}
\end{figure}

\begin{figure}[p]
    \centering
    \includegraphics[width=\linewidth]{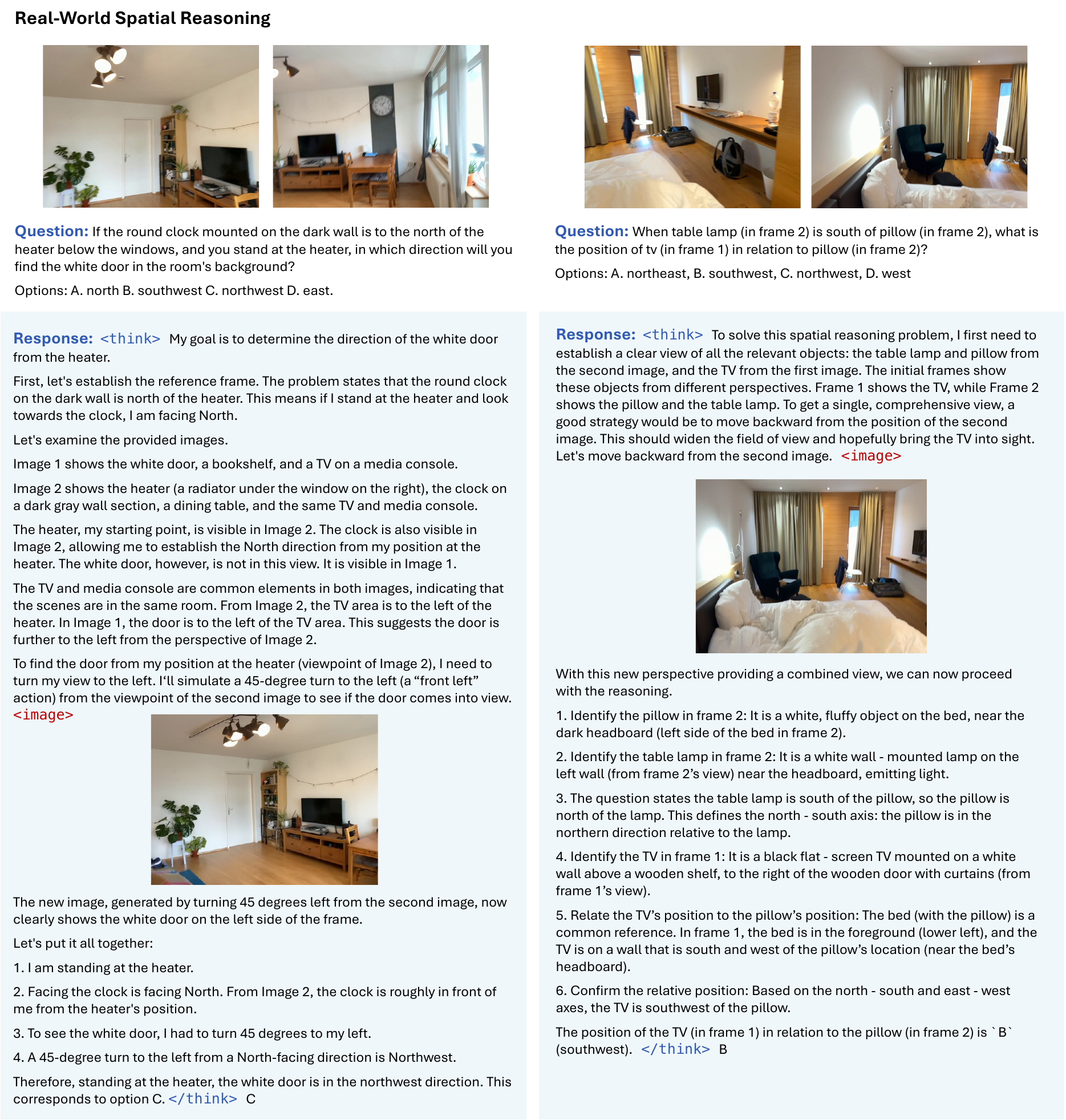}
    \caption{Examples of chain-of-thought SFT data for the real-world spatial reasoning task.}
    \label{fig:spatial_data_showcase}
\end{figure}

\clearpage

The Qwen-VL baselines are trained using \texttt{LLaMA-Factory}\footnote{\url{https://github.com/hiyouga/LLaMA-Factory}} for supervised fine-tuning (SFT) and \texttt{verl} for reinforcement learning from verifiable rewards (RLVR).

\subsection{Analytic Experiments}

\textbf{Sample efficiency.} For Figure~\ref{fig:sample_eff}, we randomly subsample either 500 or 1000 training examples. The resulting models are evaluated under two settings: (i) a hard setting with the maximum difficulty (grid size 8 and 4 folding steps, default in VisWorld-Eval), and (ii) an in-distribution setting (denoted as Normal in the figure) with randomly sampled grid sizes (3–8) and folding steps (1–4).

\textbf{Task difficulties and world model fidelity.} For Figure~\ref{fig:cube_analysis}, we generate test samples with varying cube-stack sizes (3–6), where size 6 is out-of-distribution relative to the training data. To assess world-model fidelity, we compare the generated views with the ground-truth views: for verbal world modeling, we use string pattern matching; for visual world modeling, we use Gemini 3 Pro to compare images. Since accurately inferring colors becomes particularly challenging at larger stack sizes, we evaluate only the shapes of the views and ignore color information. We also find that overall accuracy can be bottlenecked by verbal subskills (e.g., counting holes) after SFT, thus we report the accuracy of RL-trained models in Figure~\ref{fig:cube_analysis}. In contrast, RL can distract verbal world modeling capabilities, leading to invalid formats of generated symbolic matrices, thus we report world-model fidelity of SFT-trained models.

\textbf{Implicit world modeling.} For Figure~\ref{fig:implict_wm}, we supervised fine-tune (SFT) BAGEL on CoTs with implicit world modeling, in which all explicit point coordinates are replaced by the placeholder token sequence \texttt{<point>masked<point>}. After training, we extract the hidden representations at the position of the token \texttt{masked} from each transformer layer. We then split the extracted representations from different CoTs into training and validation sets with an 8:2 ratio and train a two-layer MLP (hidden size 4096) to predict the ground-truth point coordinates. Since all samples are $5\times 5$ mazes, we formulate coordinate prediction as two 5-way classification tasks (for $x$ and $y$, respectively). We compute classification accuracy for each coordinate and report the average of the two.

\section{Extended Experimental Results}
\label{app:exp}

\subsection{Full Results on MMSI-Bench}
\label{app:mmsi}

We report all scores on positional relationship tasks of MMSI-Bench in Table~\ref{tab:mmsi}.

\begin{table}[htbp]
    \setlength{\tabcolsep}{4pt}
    \caption{Full results of SFT-trained UMMs on MMSI-Bench positional relationship tasks.}
    \label{tab:mmsi}
    \centering
    \begin{tabular}{lccccccc}
\toprule
                        & \multicolumn{7}{c}{MMSI-Bench (Positional Relationship)}                    \\
Models                  & Cam.-Cam. & Obj.–Obj. & Reg.–Reg. & Cam.–Obj. & Obj.–Reg. & Cam.–Reg. & Overall\\ \midrule
Implicit WM   &           33.1  &     31.2   &  31.8     & 46.5 &   29.1  &  37.3 &  34.8 \\ 
Visual WM &            29.6  &    29.5   &  31.6 & 60.9  &   25.8   & 54.4 & 38.4 \\
\bottomrule
\end{tabular}
\end{table}

\subsection{Additional Qualitative Evaluation}

We provide additional qualitative evaluation of trained UMMs' reasoning, particularly failure cases.

\textbf{Real-world spatial reasoning.} As shown in Figure~\ref{fig:mmsi_failure_case}a, reasoning with implicit world modeling is prone to hallucinations. In contrast, visual generation (Figure~\ref{fig:mmsi_failure_case}b) yields more faithful world models, but still suffers from insufficient quality, including blurring and corrupted details. Moreover, we find that current VLMs and UMMs continue to exhibit limited understanding of positions and directions across different viewpoints. We expect that stronger base models and better-curated post-training data will enable more effective use of visual world models for spatial reasoning in future work.

\textbf{Paper folding.} As illustrated in Figure~\ref{fig:paperfolding_failure_case}, verbal reasoning about geometric symmetry is prone to hallucinations, leading to inaccurate verbal world modeling. In contrast, visual world models, benefiting from stronger prior knowledge, generate correct intermediate unfolding steps even in the presence of erroneous verbal reasoning.

\textbf{Cube 3-view projection.} As shown in Figure~\ref{fig:cube_failure_case}, visual world models are able to approximately generate novel views of cube stacks even in the challenging out-of-distribution setting with an unseen stack size of 6, indicating strong prior knowledge of spatial transformations. Nevertheless, overall task performance remains limited by subtle shape-generation errors (Figure~\ref{fig:cube_failure_case}b,d) and inaccurate color inference (Figure~\ref{fig:cube_failure_case}c). We expect these issues to be alleviated through improved post-training and stronger base models.

\begin{figure}[ptb]
    \centering
    \includegraphics[width=\linewidth]{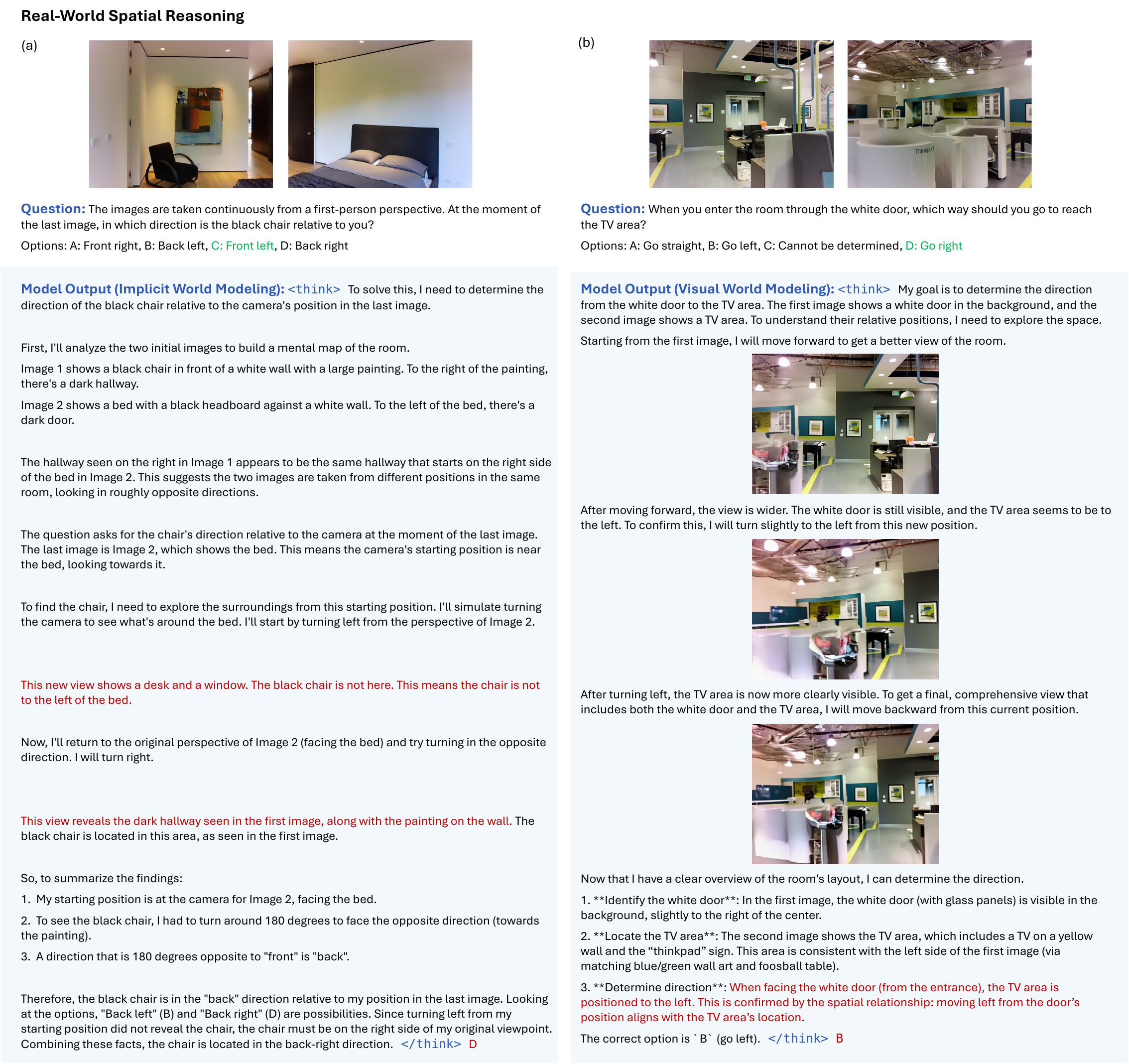}
    \caption{Showcases of reasoning generateed by post-trained UMMs in the real-world spatial reasoning task. We highlight hallucinations or incorrect reasoning steps in red.}
    \label{fig:mmsi_failure_case}
\end{figure}

\begin{figure}[ptb]
    \centering
    \includegraphics[width=\linewidth]{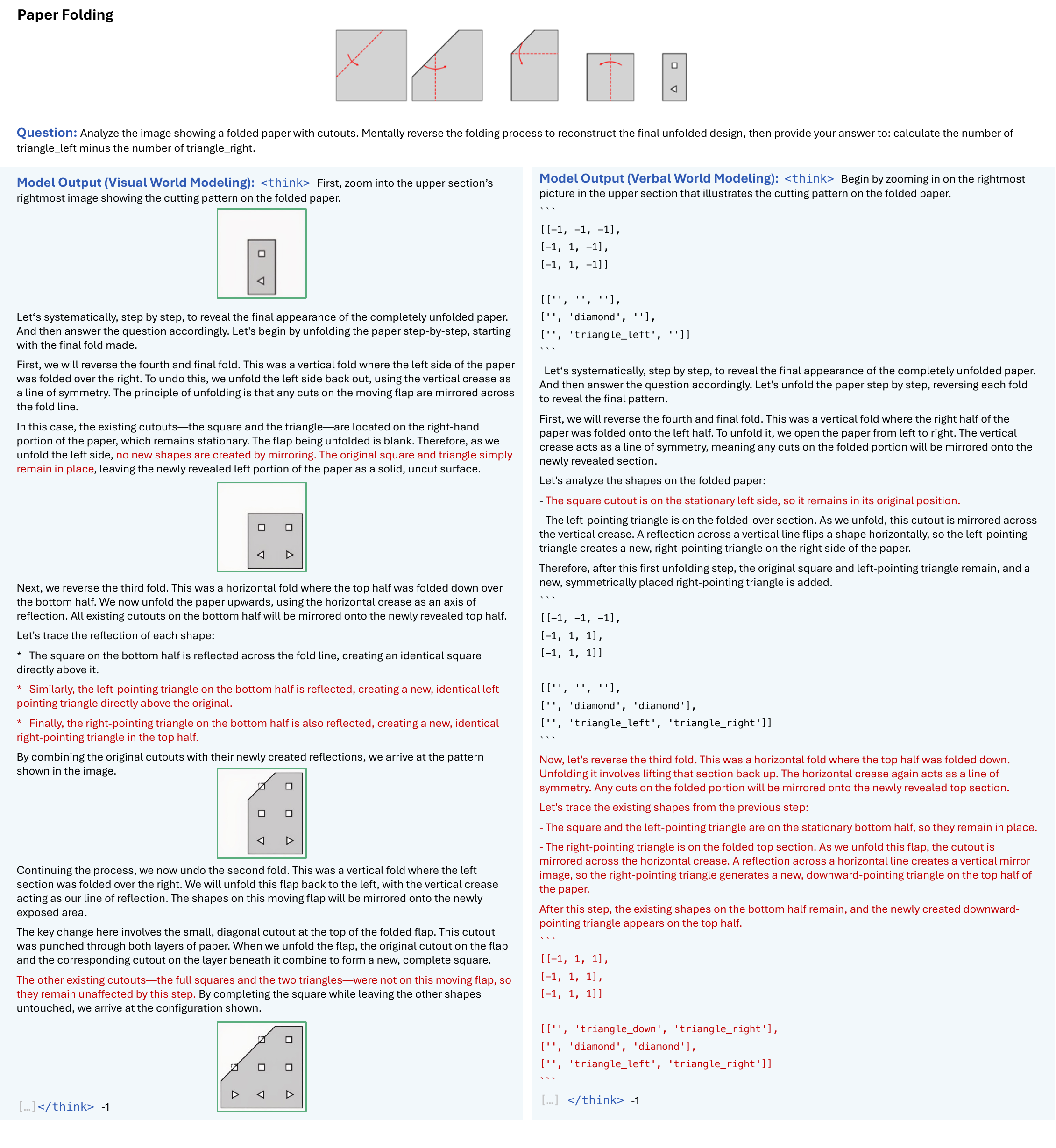}
    \caption{Showcases of reasoning generated by post-trained UMMs in the paper folding task. We highlight hallucinations or incorrect reasoning steps in red, but also mark correctly generated visual unfolding intermediate steps with green borders.}
    \label{fig:paperfolding_failure_case}
\end{figure}

\begin{figure}[ptb]
    \centering
    \includegraphics[width=\linewidth]{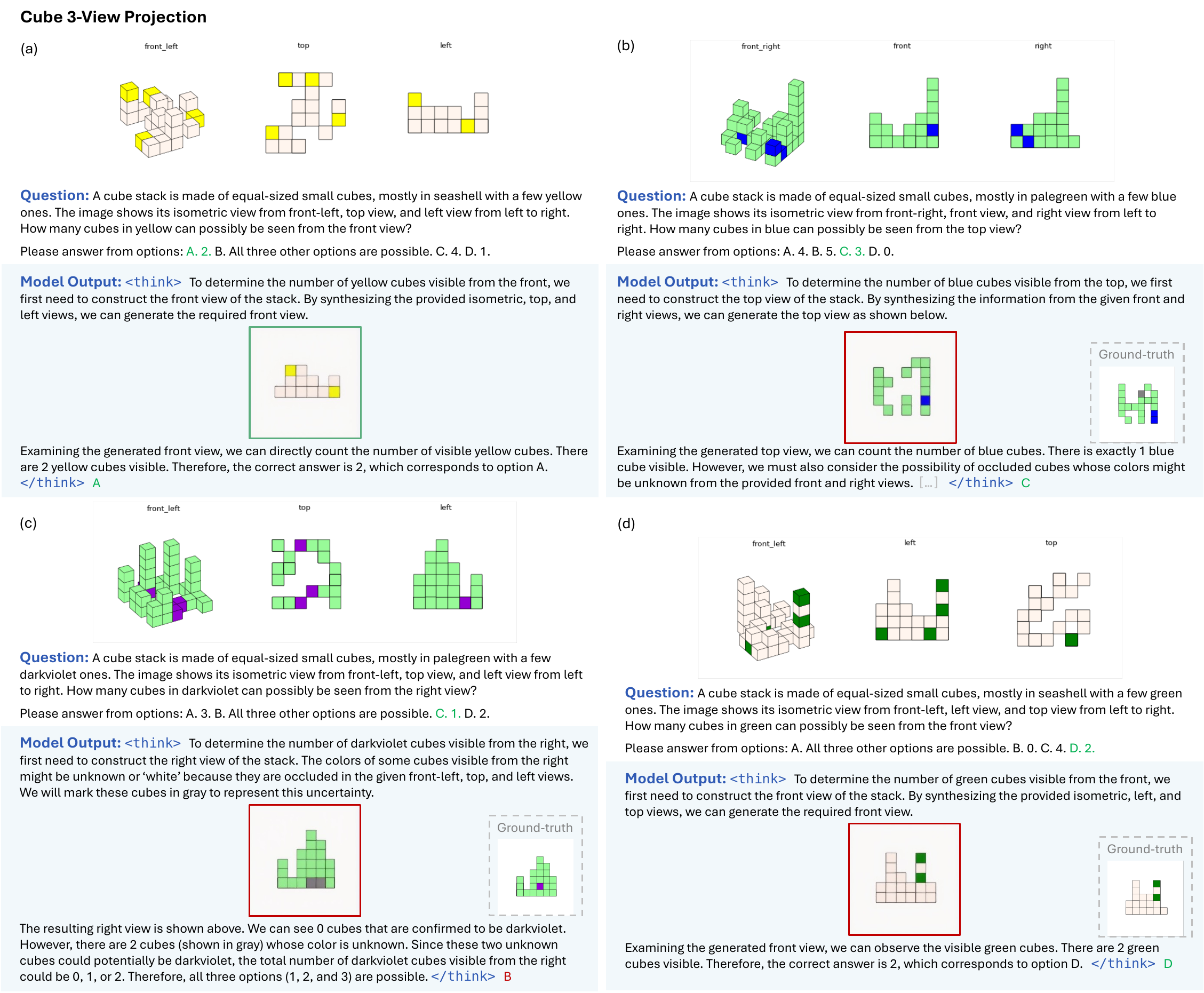}
    \caption{Showcases of reasoning generated by post-trained UMMs in the paper folding task. We mark correct and incorrect generated cube views with green and red borders, respectively. For incorrect generations, the corresponding ground-truth views are provided for reference (note that these are shown only for readers and are never provided to the models during reasoning).}
    \label{fig:cube_failure_case}
\end{figure}

\end{document}